%% file: main.tex
\documentclass{article}

\usepackage{microtype}
\usepackage{graphicx}
\usepackage{subcaption}
\usepackage{booktabs} % for professional tables

\usepackage{hyperref}

\usepackage[preprint]{icml2026}

\usepackage{placeins}

\usepackage{float}
\usepackage{makecell}
\usepackage{pythonhighlight}
\usepackage{listings}
\usepackage{amssymb}
\usepackage{graphicx}
\usepackage{amsmath}
\usepackage{amsthm}
\usepackage{amsfonts}
\usepackage{multirow}
\usepackage{makecell}
\usepackage{colortbl}
\usepackage{enumerate}
\usepackage{enumitem}
\usepackage{wrapfig}

\usepackage[capitalize,noabbrev]{cleveref}

\usepackage{multirow}
\usepackage{multicol}

\theoremstyle{plain}
\newtheorem{theorem}{Theorem}[section]

\newtheorem{lemma}[theorem]{Lemma}
\newtheorem{corollary}[theorem]{Corollary}
\theoremstyle{definition}

\theoremstyle{remark}

\newcommand{\mcL}{\mathcal{L}}
\newcommand{\mcD}{\mathcal{D}}

\newcommand{\TO}{\textbf{to}}

\usepackage[textsize=tiny]{todonotes}

\icmltitlerunning{Distributional Reinforcement Learning with Diffusion Bridge Critics}

\begin{document}

\twocolumn[
  \icmltitle{Distributional Reinforcement Learning with Diffusion Bridge Critics}



  \icmlsetsymbol{equal}{*}
  \icmlsetsymbol{Corresponding author}{$\dagger$}

  \begin{icmlauthorlist}
    \icmlauthor{Shutong Ding}{shkj,inst,equal}
    \icmlauthor{Yimiao Zhou}{shkj,inst,equal}
    \icmlauthor{Ke Hu}{shkj,inst}
    \icmlauthor{Mokai Pan}{shkj,inst}
    \icmlauthor{Shan Zhong}{uestc} \\
    \icmlauthor{Yanwei Fu}{fdu}
    \icmlauthor{Jingya Wang}{shkj,inst}
    \icmlauthor{Ye Shi}{shkj,inst,Corresponding author}
    % \icmlauthor{Firstname8 Lastname8}{shkj,comp}
  \end{icmlauthorlist}

  \icmlaffiliation{shkj}{ShanghaiTech University}
  \icmlaffiliation{inst}{InstAdapt}
  \icmlaffiliation{uestc}{University of Electronic Science and Technology of China}
  \icmlaffiliation{fdu}{Fudan University}

  \icmlcorrespondingauthor{Ye Shi}{shiye@shanghaitech.edu.cn}

  \icmlkeywords{Machine Learning, ICML}

  \vskip 0.3in
]



\printAffiliationsAndNotice{}  % no special notice (required even if empty)

\begin{abstract}
Recent advances in diffusion-based reinforcement learning (RL) methods have demonstrated promising results in a wide range of continuous control tasks. However, existing works in this field focus on the application of diffusion policies while leaving the diffusion critics unexplored. In fact, since policy optimization fundamentally relies on the critic, accurate value estimation is far more important than the policy expressiveness. Furthermore, given the stochasticity of most reinforcement learning tasks, it has been confirmed that the critic is more appropriately depicted with a distributional model. Motivated by these points, we propose a novel distributional RL method with Diffusion Bridge Critics (DBC). DBC directly models the inverse cumulative distribution function (CDF) of the Q value. This allows us to accurately capture the value distribution and prevents it from collapsing into a trivial Gaussian distribution owing to the strong distribution-matching capability of the diffusion bridge. Moreover, we further derive an analytic integral formula to address discretization errors in DBC, which is essential in value estimation. To our knowledge, DBC is the first work to employ the diffusion bridge model as the critic. Notably, DBC is also a plug-and-play component and can be integrated into most existing RL frameworks. Experimental results on MuJoCo robot control benchmarks demonstrate the superiority of DBC compared with previous distributional critic models.
\end{abstract}

\section{Introduction}

Diffusion models have recently been applied to reinforcement learning and demonstrate their superior performance in robot locomotion control and manipulation tasks~\cite{ren2024diffusion, yang2023policy, janner2022planning, wang2022diffusion, ajay2022conditional, chi2023diffusion} for their multimodality and powerful exploration capability. However, existing diffusion-based RL methods~\cite{yang2023policy, ding2024diffusion, ding2025genpo, celik2025dime, wang2024diffusion, psenka2023learning} principally focus on leveraging diffusion models to enhance the policy expressiveness, while overlooking the potential of diffusion on modeling the distributional critic. In fact, the critic plays a more important role in RL, as policy optimization are driven by the values or gradients estimated by the critic model.

% In practice, however, the critic plays a more fundamental role in policy optimization: the quality of policy updates is ultimately bounded by the accuracy and reliability of value estimation. Even highly expressive policies may fail to improve if guided by biased or miscalibrated critics. This mismatch between policy expressiveness and critic capacity becomes particularly problematic in modern reinforcement learning settings where policies are stochastic, multimodal, and trained under partial observability or noisy dynamics.

% , due to the inherent stochasticity of reinforcement learning environments, state–action values are more naturally described by a distribution rather than a single expectation~\cite{bellemare2017distributional}. Distributional reinforcement learning has shown that modeling the full return distribution can significantly improve both stability and performance, motivating a variety of approaches such as categorical critics, quantile regression, and implicit distributions~\cite{bellemare2017distributional}. However, these methods typically rely on either discrete support approximations or fixed parametric forms, which can limit their ability to capture complex, heavy-tailed, or multi-modal value distributions.
Distributional critics~\cite{bellemare2017distributional} were originally proposed to better capture the inherent stochasticity in RL, thereby yielding more accurate value estimates for policy training. The core idea of existing works~\cite{kuznetsov2020controlling, dabney2018distributional} in this area is to represent the Q-value distribution using finite discrete quantiles to approximate the Q-values at each quantile with neural networks. Although these methods can improve Q-value estimation to some extent, it remains restricted by the discrete quantile representation and the limited expressiveness of neural networks. Besides, recent works such as ~\cite{hu2025value, zhong2025flowcritic} also attempt to apply diffusion models to depict the Q-value distribution for its expressiveness. However, under the diffusion learning paradigm, the Bellman backup operator induces an unbounded accumulation of approximation errors, causing the diffusion critic to eventually collapse into a trivial Gaussian distribution.

% have explored fitting Q-value distributions using diffusion processes. While promising, directly regressing scalar Q values or modeling return samples with diffusion models often leads to degenerate solutions that implicitly assume Gaussianity or suffer from discretization errors introduced by finite diffusion steps. These issues restrict the expressiveness of diffusion-based critics and undermine their potential advantages over classical distributional methods.

To overcome these issues and obtain more accurate value estimation, we propose Diffusion Bridge Critics (DBC), a novel distributional RL method that leverages expressive diffusion bridge models to depict the Q-value distribution. Concretely, DBC learns the inverse cumulative distribution function (inverse CDF) of the Q-value distribution. This formulation aligns naturally with quantile-based distributional reinforcement learning, while avoiding discrete quantile approximations and the Gaussian degradation problem in vanilla diffusion critics. Besides, we also find that a key technical challenge in diffusion bridges for policy optimization lies in discretization errors arising from finite-time approximations. To address this challenge, we derive an analytic integral form that corrects the bias induced by discrete diffusion steps, resulting in more accurate value estimation and stable policy improvement. Notably, DBC is a plug-and-play approach, which can be seamlessly integrated into existing RL algorithms without modifying the underlying policy architecture or optimization procedure. To summarize, our contribution is threefold:
\begin{itemize}[leftmargin=4pt, rightmargin=4pt]

\item \textbf{Gaussian Degradation of Vanilla Diffusion Critic.} We identify and analyze the issue of directly employing diffusion models as critics in reinforcement learning, which leads to the degradation of itself to trivial Gaussian distributions with the Bellman backup operator. Consequently, it severely limits the expressiveness of the learned value distribution and undermines effective policy optimization.

\item \textbf{Diffusion Bridge as Distributional Critic.}
We propose Diffusion Bridge Critics (DBC), a novel distributional RL method that employs diffusion bridge models to accurately learn the inverse cumulative distribution function (CDF) of the Q-value distribution. With the anchor loss design and integral-consistent discretization technique, DBC is the first success that effectively incorporates diffusion bridge critics into distributional RL.

\item \textbf{State-of-the-art Performance.}
We evaluate DBC combined with SAC and TD3 on a range of MuJoCo robotic control tasks and compare it with other representative distributional RL methods. Results demonstrate that DBC consistently achieves state-of-the-art performance, validating the effectiveness of diffusion bridge critics and highlighting the importance of an expressive diffusion-based critic in reinforcement learning.
\end{itemize}

\section{Related Works}

% \textbf{Diffusion Bridge Models.} Diffusion bridge models have achieved significant success across a range of applications, including image translation \cite{liu20232, li2023bbdm, zhou2023denoising}, image restoration \cite{luo2023image, yue2024image, UniDB, IRBridge}, and video generation \cite{wang2024framebridge}, which overcome the limitation of Gaussian prior in standard diffusion models. Early methodologies employed Schrodinger Bridges \cite{shi2023diffusion, somnath2023aligned, tang2024simplified, gushchin2024adversarial, nobis2025fractional, zhang2025voicebridge} through direct optimization of the Kullback-Leibler (KL) divergence, stochastic interpolants \cite{albergo2023stochastic}, Doob's \textit{h}-transform \cite{zhou2023denoising, yue2024image}, and Stochastic Optimal Control (SOC) theory \cite{DBFS, UniDB} to effectively model stochastic processes that connect arbitrary pairs of distributions. 
% This work pioneers the use of diffusion bridge models as critics in reinforcement learning. The approach offers an explicit mechanism for bridging the initial distribution and the target distribution, enhancing distributional value estimation.

\textbf{Diffusion-based Reinforcement Learning.} Existing diffusion-based RL methods mostly focus on the technique of diffusion policy optimization. In offline RL, diffusion models~\cite{janner2022planning, wang2022diffusion, ajay2022conditional} are employed to approximate the behavior policy in the offline dataset and enable policy optimization by sampling actions or trajectories that remain within the support of the diffusion model. In off-policy reinforcement learning, DIPO~\cite{yang2023policy} and QSM~\cite{psenka2023learning} optimize the diffusion policy by utilizing the gradient of the Q function. DACER~\cite{wang2024diffusion} and DIME~\cite{celik2025dime} treat the diffusion policy as a black-box model and directly apply deterministic policy gradient with different techniques on entropy estimation. Besides, QVPO~\cite{ding2024diffusion}, (SDAC)~\cite{ma2025efficient} perform the policy improvement with various choices on the weighted loss. Moreover, practical techniques for fine-tuning pretrained diffusion policies have also been proposed in DPPO~\cite{ren2024diffusion}. Furthermore, FPO~\cite{mcallister2025flow} and GenPO~\cite{ding2025genpo} develop different tricks to approximate the log likelihood of the diffusion policy and try to optimize the diffusion policy with PPO loss~\cite{schulman2017proximal}. However, existing diffusion-based RL methods mostly focus on the diffusion policy and overlook how to effectively utilize the diffusion model as a powerful critic. In that case, our DBC is proposed to tackle the absence of efficient diffusion-based critics.

\textbf{Distributional Critic Learning.}
 % Diffusion and flow matching models have emerged as a promising direction for value function learning in distributional reinforcement learning, offering expressive alternatives to traditional parametric assumptions. Early explorations employed diffusion models to capture multimodal value distributions \cite{hu2025value, zhong2025mad3pg}, as well as consistency models for uncertainty-aware Q-value estimation \cite{zhang2024q}. More recently, flow matching has provided an alternative pathway, enabling iterative Q-function computation via velocity fields \cite{agrawalla2025floq} and generative modeling of return distributions \cite{zhong2025flowcritic, chen2025unleashing, dong2025value}. 
 % However, these methods either rely on iterative sampling for scalar value estimation or introduce implicit distributional assumptions when generating return samples. In contrast, our method leverages diffusion bridges to learn the inverse cumulative distribution function, enabling explicit value distribution modeling within a unified quantile framework.
Distributional reinforcement learning aims to learn the full probability distribution of value functions rather than merely estimating their expectations \cite{bellemare2017distributional, drlbook}. Early methods primarily rely on discretization or parametric assumptions. C51 discretizes the Q-value support into a fixed set of atoms and learns a categorical distribution \cite{bellemare2017distributional}. QR-DQN learns discrete quantiles via quantile regression \cite{dabney2018distributional}, and IQN further introduces implicit quantile networks that map arbitrary probability levels to corresponding Q-values, enabling more flexible distributional representations \cite{dabney2018implicit}. For continuous control, DSAC imposes Gaussian assumptions on value distributions within the SAC framework \cite{dsac, dsac2}, while TQC improves stability through ensembles of quantile critics \cite{kuznetsov2020controlling}. Despite their effectiveness, these methods remain constrained by discrete supports or fixed parametric forms, limiting their capacity to capture complex value distribution structures. Very recently, diffusion and flow matching models have emerged as more expressive alternatives for value distribution learning. Diffusion models learn multimodal value distributions through iterative denoising \cite{hu2025value, zhong2025mad3pg}, consistency models enable efficient uncertainty estimation via single-step sampling \cite{zhang2024q}, and flow matching approaches either compute Q-functions iteratively through learned velocity fields \cite{agrawalla2025floq} or directly model return distributions generatively \cite{zhong2025flowcritic, chen2025unleashing, dong2025value}.

Nevertheless, existing generative methods either rely on iterative sampling for scalar value estimation or introduce implicit distributional assumptions when generating return samples. In contrast, our method leverages diffusion bridges to learn the inverse cumulative distribution function, enabling explicit value distribution modeling within a unified quantile framework.

\begin{figure*}[ht]
    \centering
    \includegraphics[width=\linewidth]{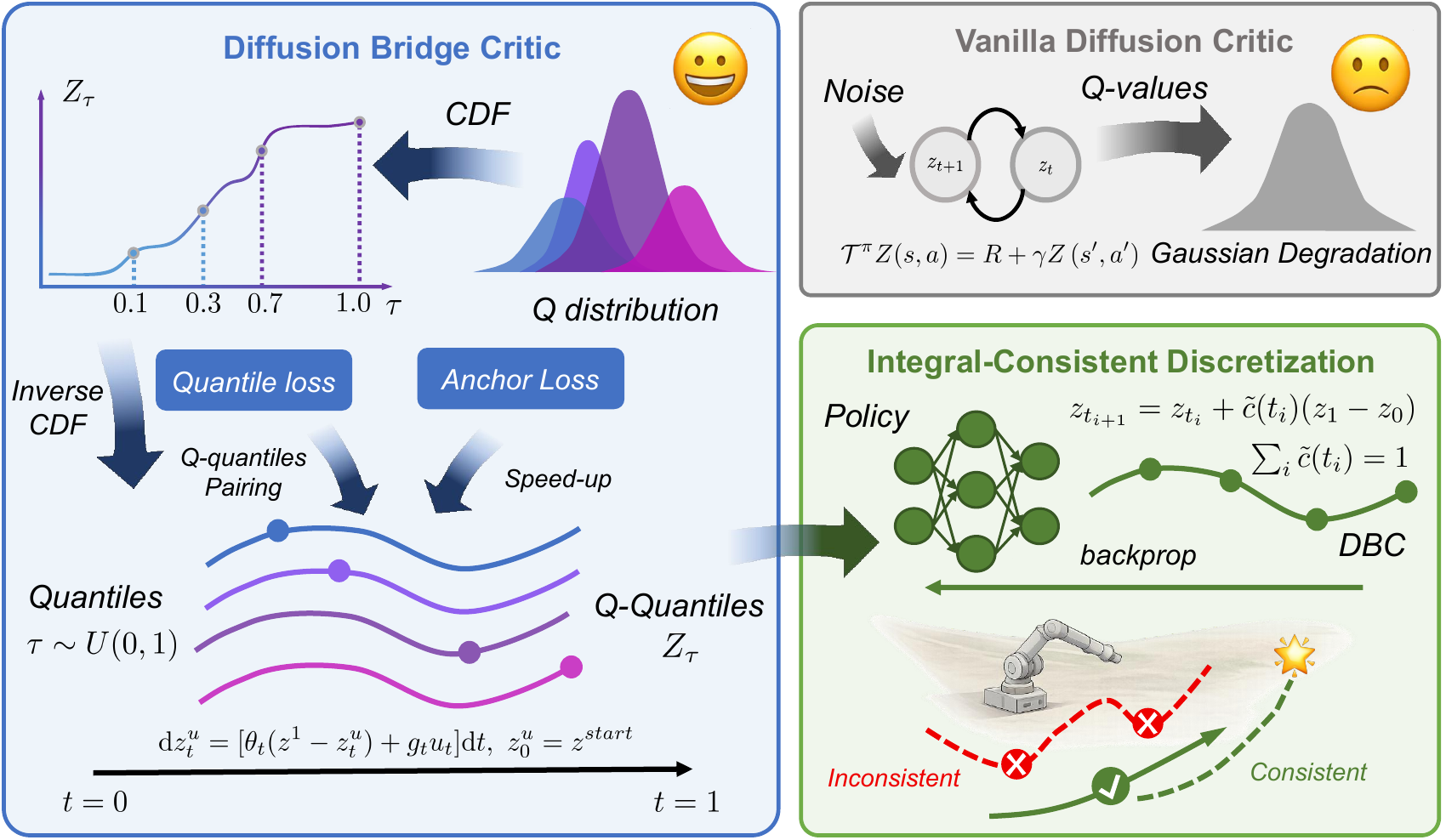}
    \caption{The training pipeline of Diffusion Bridge Critics. Compared with Vanilla Diffusion Critic, DBC explicitly models the inverse cumulative distribution function (CDF) of Q-values and resolves the Gaussian Degradation Problem. Besides, the design of the integral-consistent discretization technique is also developed for accurate value estimation and stable policy optimization.}
    \label{fig:dbc}

\end{figure*}

\section{Preliminaries}
\subsection{Distributional Reinforcement Learning}

Consider a Markov Decision Process (MDP) defined by the quintuple \((\mathcal{S}, \mathcal{A}, \mathcal{P}, \mathcal{R}, \gamma)\). Here, \(\mathcal{S}\) and \(\mathcal{A}\) denote the state and action spaces, respectively, with state \(\boldsymbol{s} \in \mathcal{S}\) and action \(\boldsymbol{a} \in \mathcal{A}\) both represented as vectors. The state transition probability density is given by \(\mathcal{P}: \mathcal{S} \times \mathcal{A} \to \Delta(\mathcal{S})\), \(\mathcal{R}(\boldsymbol{s}, \boldsymbol{a})\) is a stochastic reward function, and \(\gamma \in [0, 1)\) is the discount factor. A policy \(\pi: \mathcal{S} \to \mathcal{A}\) maps states to actions.

The random return \(Z^\pi(\boldsymbol{s}, \boldsymbol{a})\) is defined as the discounted sum of future rewards \(r_k \in \mathcal{R}\):
\begin{equation}
\begin{aligned}
&Z^\pi(\boldsymbol{s}, \boldsymbol{a}) = \sum_{k=0}^{\infty} \gamma^k r_k, \quad \boldsymbol{s}_0=\boldsymbol{s}, \ \boldsymbol{a}_0=\boldsymbol{a}, 
\\ 
&\boldsymbol{a}_k \sim \pi(\cdot|\boldsymbol{s}_k), \ \boldsymbol{s}_{k+1} \sim \mathcal{P}(\cdot|\boldsymbol{s}_k, \boldsymbol{a}_k).
\end{aligned}
\end{equation}
In contrast to traditional reinforcement learning, which estimates the expected return \(Q^\pi(\boldsymbol{s}, \boldsymbol{a}) = \mathbb{E}[Z^\pi]\), distributional reinforcement learning aims to model the full distribution of the one-dimensional random variable \(Z^\pi\). The distributional Bellman operator \(\mathcal{T}^\pi\) is defined as:
\begin{equation}
\mathcal{T}^\pi Z(\boldsymbol{s}, \boldsymbol{a}) \stackrel{D}{=} \mathcal{R}(\boldsymbol{s}, \boldsymbol{a}) + \gamma Z(\boldsymbol{s}', \pi(\boldsymbol{s}')),\,\boldsymbol{s}' \sim \mathcal{P}(\cdot|\boldsymbol{s}, \boldsymbol{a}),
\label{eq:build_target}
\end{equation}
where \(\stackrel{D}{=}\) denotes equality in distribution. Under the \(p\)-Wasserstein distance, \(\mathcal{T}^\pi\) is a \(\gamma\)-contraction operator [citation], which guarantees the existence and uniqueness of its distributional fixed point.

We characterize the distribution via its quantile function \(F^{-1}_{Z(\boldsymbol{s}, \boldsymbol{a})}(\tau)\), where \(\tau \in [0,1]\) is the quantile level. 
% Assuming the return distribution is continuous, its quantile function is monotonically non-decreasing on \([0,1]\):
% \begin{equation}
% \forall\,0 \leq \tau_1 \leq \tau_2 \leq 1, \quad F_Z^{-1}(\tau_1) \leq F_Z^{-1}(\tau_2).
% \end{equation}
% In practice, we treat \(\tau\) as a conditional input and learn an approximating function \(f_\theta(\boldsymbol{s}, \boldsymbol{a}, \tau)\) by minimizing the quantile regression loss:
\begin{equation}
\begin{aligned}
\mathcal{L}_{\text{QR}}(\theta) 
&= \mathbb{E}_{(\boldsymbol{s}, \boldsymbol{a})\sim \mathcal{D}, \ \tau \sim U([0,1])} \left[ \rho_\tau^\kappa \big( y - f_\theta(\boldsymbol{s}, \boldsymbol{a}, \tau) \big) \right],\\
\rho^\kappa_\tau(u) 
&= |\tau - \mathbb{I}(u<0)|\cdot \mcL_\kappa^{\text{Huber}}(u)
% \mcL_\kappa^{\text{Huber}}(u) 
% &= 
% \begin{cases} 
% \frac{1}{2}u^2 & \text{if } |u| \le \kappa \\
% \kappa(|u| - \frac{1}{2}\kappa) & \text{otherwise}
% \end{cases}
\end{aligned}
\end{equation}
where \(y\) is the target sample and \(\mcL_\kappa^{\text{Huber}}\) is the huber loss with parameter $\kappa$.

 We provide a detailed derivation of the empirical quantile and proofs of the related results in \textbf{Appendix}~\ref{appendix:empirical_quantile}.

% TODO

% \textbf{Diffusion Bridge Models.} Diffusion bridge models have achieved significant success across a range of applications, including image translation \cite{liu20232, li2023bbdm, zhou2023denoising}, image restoration \cite{luo2023image, yue2024image, UniDB, IRBridge}, and video generation \cite{wang2024framebridge}, which overcome the limitation of Gaussian prior in standard diffusion models. Early methodologies employed Schrodinger Bridges \cite{shi2023diffusion, somnath2023aligned, tang2024simplified, gushchin2024adversarial, nobis2025fractional, zhang2025voicebridge} through direct optimization of the Kullback-Leibler (KL) divergence, stochastic interpolants \cite{albergo2023stochastic}, Doob's \textit{h}-transform \cite{zhou2023denoising, yue2024image}, and Stochastic Optimal Control (SOC) theory \cite{DBFS, UniDB} to effectively model stochastic processes that connect arbitrary pairs of distributions. 
% This work pioneers the use of diffusion bridge models as critics in reinforcement learning. The approach offers an explicit mechanism for bridging the initial distribution and the target distribution, enhancing distributional value estimation.

\subsection{Diffusion Bridge} \label{sec:prelim_soc}
Diffusion bridge models have achieved significant success across a range of applications, including image translation \cite{liu20232, li2023bbdm, zhou2023denoising, albergo2023stochastic}, image restoration \cite{luo2023image, yue2024image, UniDB, IRBridge}, and video generation \cite{wang2024framebridge}, which overcome the limitation of Gaussian prior in standard diffusion models. UniDB \cite{UniDB} is a unified framework demonstrating that \textit{h}-transform-based models~\cite{zhou2023denoising, yue2024image} can be unified through the lens of Stochastic Optimal Control (SOC) theory. To model the transition starting from a sample $z^{\text{start}}$ in prior distribution to $z^{\text{end}}$data distribution, UniDB \cite{UniDB} constructed the following Linear Quadratic Optimal Control problem \cite{SOCtheory, chen2024generativemodelingphasestochastic} as
\begin{equation}\label{eq:unidb_soc_ode}
\begin{gathered}
\min_{u_t} \ \int_{0}^{1} \frac{1}{2} \|u_t\|_2^2 \mathrm{d}t + \frac{\gamma}{2} \| z^u_1 - z^{\text{end}} \|_2^2 \\
\text{s.t.} \ \mathrm{d} z^u_t = [\theta_t (z^{\text{end}} - z_t^u) + g_t u_t] \mathrm{d}t, \ z^u_0 = z^{\text{start}}, % \ x^u_1 = z^{\text{end}}, 
\end{gathered}
\end{equation}
% \min_{\boldsymbol{u}_{t}} \ \int_{0}^{1} \frac{1}{2} \|\boldsymbol{u}_{t}\|_2^2 \mathrm{d}t + \frac{\gamma}{2} \| \boldsymbol{x}^u_1 - \boldsymbol{x}_1 \|_2^2 \\
% \text{s.t.} \ \mathrm{d}\boldsymbol{x}^u_t = [\theta_t (\boldsymbol{x}_1 - \boldsymbol{x}^u_t) + g_t \boldsymbol{u}_{t}] \mathrm{d}t, \ \boldsymbol{x}^u_0 = \boldsymbol{x}_0, 
% 
where $x^u_t$ is the controlled diffusion process, $u_t$ is the controller, $\int_{0}^{1} \frac{1}{2} \|u_t\|_2^2 \mathrm{d}t$ is the transient cost, $\frac{\gamma}{2} \| x^u_1 - z^{\text{end}} \|_2^2$ is the terminal cost with its penalty coefficient $\gamma$, $\theta_t$ and $g_t$ are two scalar-valued functions with a relationship $g_t^2 = 2 \lambda^2 \theta_t$ where the steady variance level $\lambda^2$ is a given constant. Here we directly consider the deterministic optimal control problem~\eqref{eq:unidb_soc_ode} with fixed endpoints ($\gamma \rightarrow \infty$ to make the controlled dynamics precisely converge to $z^{\text{end}}$) \cite{chen2024generativemodelingphasestochastic} in UniDB because in the context of distributional reinforcement learning, the relationship between the boundary samples $z^{\text{start}}$ and $z^{\text{end}}$ is deterministic. According to Pontryagin Maximum Principle \cite{SOCtheory}, UniDB provided the solution to the SOC problem Eq.~\eqref{eq:unidb_soc_ode} and the corresponding transition $z_t$ between $z^{\text{start}}$ and $z^{\text{end}}$ as
\begin{equation}\label{eq:optimal_solution}
\begin{gathered}
u_t^* = g_t \frac{e^{-2\bar{\theta}_{t:1}}}{\bar{\sigma}_{t:1}^2}(z^{\text{end}} - z_t^u), \\
z_t^u = \xi(t) z^{\text{start}} + (1 - \xi(t)) z^{\text{end}}, \ \xi(t) = e^{-\bar{\theta}_{0:t}}\frac{\bar{\sigma}_{t:1}^2}{\bar{\sigma}_{0:1}^2}, 
\end{gathered}
\end{equation}
where $\bar{\theta}_{s:t} = \int_s^t \theta_z \mathrm{d} z$ and $\bar{\sigma}_{s:t}^2 = \lambda^2 (1 - e^{-2\bar{\theta}_{s:t}})$. For more details, please refer to UniDB \cite{UniDB}.

\section{Diffusion Bridge Critic}\label{sec:DBC}

In this section, we first reveal the Gaussian degradation problem in vanilla diffusion critic. To address this issue and sufficiently utilize the expressiveness of the diffusion bridge model, we propose the \textbf{Diffusion Bridge Critic} (DBC). With the quantile loss, it models the Q-value distribution as a "bridge" between quantiles \(\tau\) and the corresponding Q-quantiles \(Z_\tau\). Besides, we also design an anchor loss to reduce the training variance and speed up the convergence. Moreover, we further investigate the impact of DBC on policy improvement. Since the implementation of DBC is actually the discretization of diffusion bridges, we propose an integral-consistent discretization technique to calibrate the parameters of DBC, thereby providing accurate value estimates for policy improvement. The full pipeline of DBC is shown in Figure~\ref{fig:dbc}.
\input{algorithm/train_brief}
\subsection{Limitation of Vanilla Diffusion Critics}
As we mentioned before, diffusion critics tend to collapse to a trivial Gaussian distribution due to their distribution learning paradigm. Intuitively, the approximation errors in the TD target can be continuously accumulated by the Bellman operator under the diffusion-critic training paradigm, causing the learned distribution to gradually converge toward a Gaussian distribution, as suggested by the central limit theorem.
\begin{theorem}[Gaussian Degradation of Diffusion Critics]
    % \textbf{Gaussian Degradation of Diffusion Critics.} Diffusion Critics $f_\theta$ finally degrades into a Gaussian distribution $\mathcal{N}(Q(s,a), \sigma^2)$ with the Bellman update:
    Diffusion Critics $f_\theta$ finally degrades into a Gaussian distribution $\mathcal{N}(Q(s,a), \sigma^2)$ under the Bellman backup operator:
    \[
     \mathcal{T}^{\pi} Z(s, a)=R(s, a)+\gamma Z\left(s^{\prime}, a^{\prime}\right), Z\sim f_\theta(s,a).
    \]
    \label{thm:gaussian_degradation}
    \vspace{-2em}
\end{theorem}
The proof can be referred to in \textbf{Appendix} \ref{appendix:proof}.

Besides, we also provide a toy example to illustrate that directly applying diffusion and flow matching methods to distributional RL leads to a tendency toward Gaussian distributions as bootstrapping training iterations progress in Figure~\ref{fig:fm_vs_DBC}. In contrast, DBC avoids this issue by explicitly modeling the relationship between \(z^{\text{start}}\) and \(z^{\text{end}}\) since it fits the mapping of the inverse CDF of the Q-value distribution instead of directly modeling the Q-value distribution. We provide detailed analyses in \textbf{Appendix} \ref{appendix:proof}.

\begin{figure*}[ht]
    \centering
    \includegraphics[width=1.0\linewidth]{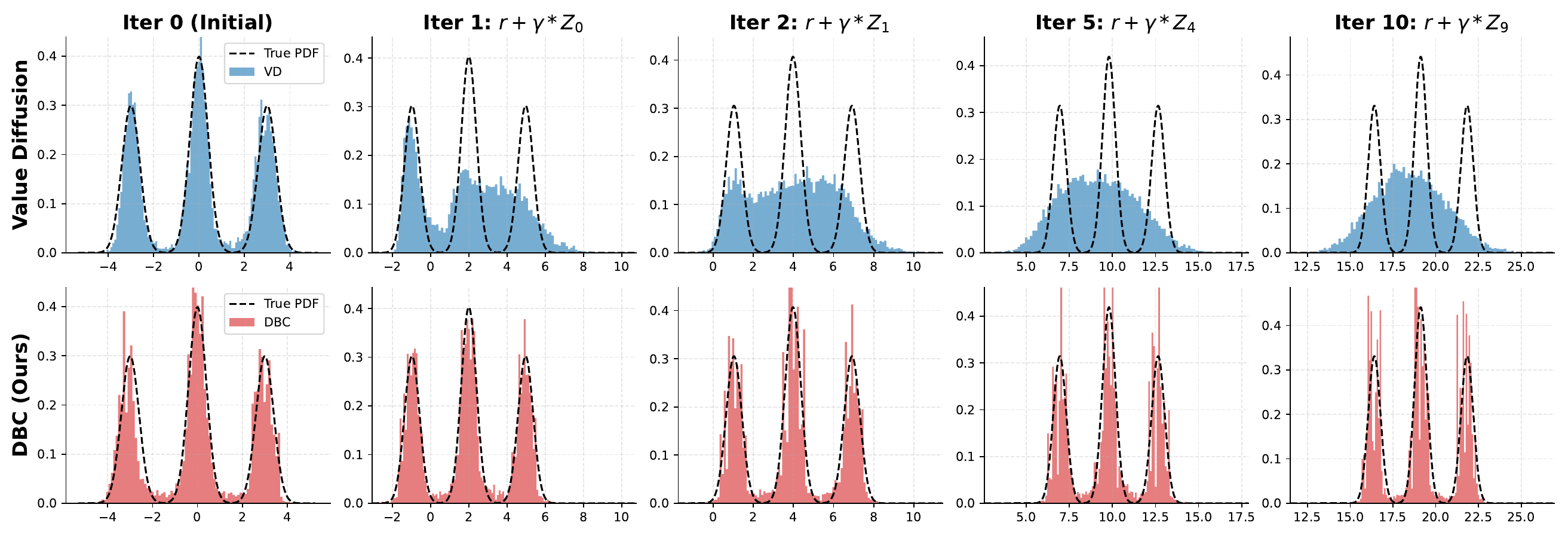}
    \caption{Top row: Value diffusion fits the iterative target \(r+\gamma \cdot Z_{k-1}^{\mathrm{FM}}\). Bottom row: DBC fits the iterative target \(r+\gamma \cdot Z_{k-1}^{\mathrm{DBC}}\). DBC introduces quantile conditioning \(\tau\), which consequently preserves multimodality of the distribution under iterative Bellman drift \(Z_k \leftarrow r+\gamma Z_{k-1}\). We takes $100$ inner training steps to fit $Z_k$ for each method in iteration and 10000 steps to fit the initial distribution in Iter 0.}
    \label{fig:fm_vs_DBC}

\end{figure*}

\subsection{Training Process of DBC}
\label{sec:empirical_anchor_training}

To resolve the Gaussian degradation problem, we propose diffusion bridge critics, which adopts a quantile-function parameterization for the distributional critic. Given a state--action pair $(\boldsymbol{s},\boldsymbol{a})$, a quantile level $\tau\in(0,1)$, and a bridge time $t\in[0,1]$, the online network directly predicts the data endpoint
\begin{equation}
\hat{z}_{\tau,t} = f_{\theta}(z_t,t,\tau,\boldsymbol{s},\boldsymbol{a}),
\end{equation}
where $z_t$ is an intermediate variable along the diffusion bridge trajectory Eq.~\eqref{eq:optimal_solution}. Since our training objectives are defined at the data endpoint $z^{\text{end}}$ (i.e., return particles), we employ a \emph{data-prediction} parameterization so that the supervision acts directly on $z^{\text{end}}$, avoiding time-scale reweighting and error propagation induced by additional parameter transformations.

For each critic update, we use a target network to generate a set of target return particles by Eq.~\eqref{eq:unidb_soc_ode}. Concretely, the target critic samples $\{z^{\mathrm{target}}_j\}_{j=1}^{K_{\text{tgt}}}$ at the next state--action pair $(\boldsymbol{s}',\boldsymbol{a}')$ through the bridge Eq.~\eqref{eq:unidb_soc_ode}, and we apply a Bellman transform to obtain
\begin{equation}\label{eq:target_particle}
\mathcal{Y}=\{y_j\}_{j=1}^{{K_{\text{tgt}}}},
\quad
y_j := r + \gamma z^{\mathrm{target}}_j.
\end{equation}
Given $\mathcal{Y}$, we jointly optimize two losses: a standard quantile loss and an \textit{anchor loss}.

Following prior quantile-based distributional RL methods, we regress each predicted quantile to the target particle distribution using the huberized quantile loss:
\begin{equation}\label{eq:loss:quantile_loss}
\mathcal{L}_{t}^{\mathrm{quantile}}
=
\sum_{j=1}^{{K_{\text{tgt}}}}\sum_{i=1}^{K_{\text{onl}}}
\rho_{\tau_i}^{\kappa}\!\Bigl(y_j - f_\theta(z_t,t,\tau_i,\boldsymbol{s},\boldsymbol{a})\Bigr),
\end{equation}
where $\{\tau_i\}_{i=1}^{K_{\text{onl}}}\subset(0,1)$ are quantile levels, $K_{\text{onl}}$ represents the number of samples generated by nline network. This objective exploits all particles to capture the distributional shape. However, the supervision for each $\tau_i$ is \emph{implicit}, the network must infer the quantile location through comparisons to all $y_j$, and the resulting gradient estimator can have high variance, which may slow down the early formation of a well-structured quantile function and destabilize training.

\paragraph{Anchor loss.}
To provide a more direct and lower-variance learning signal, we introduce an estimator for quantile. Let the order statistics of $\mathcal{Y}$ be
\begin{equation}
y_{(1)}\le y_{(2)}\le \cdots \le y_{({K_{\text{tgt}}})}.
\end{equation}
We define the sample $\tau$-quantile by
\begin{equation}\label{eq:empirical_quantile}
y^{\mathrm{sample}}(\tau) := y_{(\lceil {K_{\text{tgt}}}\tau\rceil)}.
\end{equation}
Importantly, $y^{\mathrm{sample}}(\tau)$ is not a heuristic label: it can be characterized as a \textbf{minimizer of Eq.~\eqref{eq:loss:quantile_loss}} with $\kappa=0$, i.e., it is the explicit solution obtained by solving Eq.~\eqref{eq:loss:quantile_loss} with fixed samples $\mathcal{Y}$. We provide a formal proof of this relationship and further properties of empirical quantiles in Appendix~\ref{appendix:empirical_quantile}. Thus, we can build $z_t$ by $y^{\text{sample}}(\tau)$ and $\tau$. For $t=0$, we directly set $z_t = \tau$. Using these anchors, we impose an explicit pointwise regression term with a robust Huber penalty:
 \begin{equation}\label{eq:loss:anchor}
\mathcal{L}_{t}^{\mathrm{Anchor}}
=
\sum_{i=1}^{K_{\text{onl}}}
\mathcal{L}_{\kappa}^{\mathrm{Huber}}
\Bigl(
f_\theta(z_t,t,\tau_i,\boldsymbol{s},\boldsymbol{a})
-
y^{\mathrm{sample}}(\tau_i)
\Bigr).
\end{equation}
This anchor loss provides direct supervision for each $\tau_i$, reducing optimization ambiguity and gradient variance, and accelerating the early shaping of the quantile function. While it does not impose a hard monotonicity constraint, the anchored targets $\tau_i \mapsto y^{\mathrm{sample}}(\tau_i)$ are naturally ordered, which encourages more consistent quantile structure and mitigates quantile crossing in practice. 

We train the critic with the weighted combination
\begin{equation}\label{eq:loss:total}
\mathcal{L}_{t}
=
\mathcal{L}_{t}^{\mathrm{quantile}}
+
w\cdot \mathcal{L}_{t}^{\mathrm{Anchor}},
\end{equation}
where $w\ge 0$ controls the anchoring strength. However, relying solely on the anchor loss is not recommended, as the anchor \textbf{provides a biased target in expectation}, with the bias magnitude on the order of  \(o\left(\frac{1}{\sqrt{K^{\mathrm{tgt}}}}\right)\). We provide a detailed analysis and discussion in Theorem~\ref{thm:bias_sample_quantile}. For conditioning, we encode both the quantile level \(\tau\) and the bridge time \(t\) using cosine embeddings and inject them into the critic. To mitigate value overestimation during critic updates, we adopt a DropTop aggregation strategy, which discards the largest quantile estimates and aggregates the remaining ones to obtain a conservative value estimate, following prior work on quantile-based critics~\cite{kuznetsov2020controlling}. 

For gradient-based methods such as SAC and TD3, whose policy updates rely on gradients of the return estimate, the backpropagation is performed through the entire sampling procedure.

Although the diffusion bridge defines a deterministic trajectory from $z^{\text{start}}$ to $z^{\text{end}}$, naive Euler discretization and iterative sampling may introduce\textbf{ numerical errors and endpoint bias}. We discuss this issue and the corresponding remedies in the next section.

\subsection{Integral-consistent Discretization of DBC}
\label{sec:integral-consistent}

In the previous section, we described how DBC is trained. However, Eq.~\eqref{eq:unidb_soc_ode} cannot be directly used to generate rewards for the policy. In practice, this equation must be discretized in order to be iteratively evaluated. Unfortunately, commonly used discretization schemes introduce \emph{endpoint bias} when applied to diffusion bridges. This bias causes the policy network to be optimized with respect to wrong reward signals, which can lead to training instability or even collapse.  
To address this issue, we propose an \emph{integral-consistent discretization} scheme in this section. A brief description of the proposed algorithm is provided in Algorithm~\ref{alg:dbc-ode-sampling}. 
% In policy training, inference efficiency is critical: the critic must generate return samples with a strictly limited number of function evaluations.
% In our framework, both target construction and policy evaluation require numerically integrating the diffusion-bridge dynamics from the prior endpoint $z^{\text{start}}$ to the data endpoint $z^{\text{end}}$.
% A naive discretization of the underlying continuous-time dynamics (e.g., Euler updates), however, generally introduces \emph{systematic} integration errors, causing the discretized trajectory to violate the intended boundary condition at the endpoint.
% We refer to this phenomenon as \emph{endpoint bias}.
% In this section, we formalize the source of endpoint bias in diffusion-bridge critics and present a remedy based on \emph{integral-consistent discretization}, which preserves the global geometry of the continuous dynamics while retaining fast inference.

To make the source of endpoint bias explicit, we rewrite the UniDB solution in Eq.~\eqref{eq:optimal_solution} and substitute it into the controlled dynamics in Eq.~\eqref{eq:unidb_soc_ode}.
This yields the following deterministic ODE along the bridge:
\begin{equation}
\frac{\mathrm{d} z_t}{\mathrm{d} t}
=
c(t)\,(z^{\text{end}} - z^{\text{start}}),
\quad
c(t)
=
\left(
\theta_t
+
g_t^{2}
\frac{e^{-2\bar{\theta}_{t:1}}}{\bar{\sigma}_{t:1}^{2}}
\right)
\xi(t),
\end{equation}
where $z_t \in \mathbb{R}$ denotes the scalar return value at time $t$, and $(z^{\text{start}},z^{\text{end}})$ are the prescribed boundary endpoints.
For the continuous-time trajectory to satisfy $z_0=z^{\text{start}}$ and $z_1=z^{\text{end}}$, the velocity coefficient must satisfy the global constraint
\begin{equation}
\int_{0}^{1} c(t)\,\mathrm{d}t = 1.
\end{equation}

\input{algorithm/sampling}

Let $0=t_0<t_1<\cdots<t_M=1$ be an arbitrary time partition.
A conventional Euler discretization applies the update
\begin{equation}
z_{t_{i+1}}
=
z_{t_i}
+
c(t_i)\,\Delta t_i\,(z^{\text{end}}-z^{\text{start}}),
\quad
\Delta t_i = t_{i+1}-t_i.
\end{equation}
After $M$ steps, the accumulated displacement becomes
\begin{equation}
z_{t_M}-z_{t_0}
=
\left(
\sum_{i=0}^{M-1} c(t_i)\Delta t_i
\right)
(z^{\text{end}}-z^{\text{start}}).
\end{equation}
For finite $M$, the discrete sum $\sum_i c(t_i)\Delta t_i$ generally deviates from the continuous integral $\int_0^1 c(t)\mathrm{d}t=1$, resulting in $z_{t_M}\neq z^{\text{end}}$. We refer to this discrepancy as endpoint bias.
Importantly, this discrepancy is not stochastic numerical noise but a systematic violation of the global boundary constraint. We provide a detailed analysis and empirical illustration of this endpoint bias in Appendix~\ref{appendix:integral-consistent}.

We now show that endpoint bias can be eliminated entirely by enforcing integral consistency at the discretization level.

\input{table/experiments}
\begin{theorem}[Endpoint Consistency via Integral-Consistent Discretization]
\label{thm:integral-consistent}
Let $z_t \in \mathbb{R}$ follow the ODE
\begin{equation}
\frac{\mathrm{d}z_t}{\mathrm{d}t} = c(t)\,(z^{\text{end}}-z^{\text{start}}),
\end{equation}
where $z^{\text{start}},z^{\text{end}}\in\mathbb{R}$ are fixed endpoints and $c:[0,1]\to\mathbb{R}$ is integrable with
\(
\int_0^1 c(t)\mathrm{d}t = 1.
\)
For any partition $0=t_0<t_1<\cdots<t_M=1$, define the discrete update by exact interval integration:
\begin{equation}
z_{t_{i+1}}
=
z_{t_i}
+
\left(
\int_{t_i}^{t_{i+1}} c(t)\,\mathrm{d}t
\right)
(z^{\text{end}}-z^{\text{start}}).
\end{equation}
Then the resulting discrete trajectory satisfies $z_{t_M}=z^{\text{end}}$ whenever $z_{t_0}=z^{\text{start}}$, independent of $M$ and the choice of partition.
\end{theorem}

We note the following identity between the interpolation coefficient $\xi(t)$ and the velocity coefficient $c(t)$:
\begin{equation}
c(t) = -\frac{\mathrm{d}\xi(t)}{\mathrm{d}t},
\label{eq:relation_interpolate}
\end{equation}
with boundary conditions $\xi(0)=1$ and $\xi(1)=0$.
A formal proof is provided in Appendix~\ref{appendix:integral-consistent}.

As a direct consequence, it leads to the following corollary.

\begin{corollary}[Integral-consistent discretization for DBC]
\label{cor:dbc-integral-consistent}
For any partition $0=t_0<t_1<\cdots<t_M=1$, set $\tilde{c}(t_i) := \xi(t_i)-\xi(t_{i+1})$, the update rule
\begin{equation}
\label{eq:update_rule}
z_{t_{i+1}}
=
z_{t_i}
+
\tilde{c}(t_i)(z^{\text{end}}-z^{\text{start}})
\end{equation}
is integral-consistent and guarantees strict endpoint consistency $z_{t_M}=z^{\text{end}}$ for any number of steps $M$.
\end{corollary}

This discretization exactly preserves the global geometry of the continuous bridge while requiring only evaluations of $\xi(t)$, without numerical quadrature or step-size tuning.
As a result, DBC maintains stable generative semantics and strict endpoint consistency even with as few as $M=5$ inference steps, which is critical for low-latency online reinforcement learning.
We provide a concise for our sampling process in Algorithm~\ref{alg:dbc-ode-sampling}.

\paragraph{Policy training.}
After the critic is fully trained, we use the online network to generate \(K_{\text{onl}}\) return samples via Algorithm~\ref{alg:dbc-ode-sampling}.
These samples are averaged to obtain an expected return estimate \(q_{\text{onl}}\), which is then provided to the actor for policy optimization.
We provide a brief overview of our method in Algorithm~\ref{alg:dbq-training-brief}.

\section{Experiments}
In this section, we systematically evaluate the performance of DBC on multiple MuJoCo continuous-control benchmarks. We first compare DBC against several representative critic baselines to quantify its overall performance gains. We then assess DBC as a plug-and-play critic module by integrating it into different actor algorithms, thereby validating its generality across actor backbones. Finally, we conduct two ablation studies to analyze how key hyperparameters affect performance and training stability, including the flow steps \(M\) and the anchor loss weight \(w\). 
% These experiments are designed to answer the following questions: (i) How does DBC perform overall? (ii) Does DBC consistently improve performance when used as a plug-and-play critic module across different actor backbones? (iii) How should \(M\) and \(w\) be selected to balance performance, computational cost, and training stability?

\subsection{Comparative Evaluation}
We benchmark DBC against several critic baselines, including SAC~\citep{haarnoja2018soft}, DSAC~\citep{dsac}, IQN~\citep{dabney2018implicit}, TQC~\citep{kuznetsov2020controlling}, Value Diffusion (VD)~\citep{hu2025value}, and Value Flows (VF)~\citep{dong2025value}, all implemented within our codebase. Brief descriptions of these baselines are provided in Appendix~\ref{appendix:baselines}. Table~\ref{tab:main:experiments} presents the main results. For a controlled comparison, all baseline methods use a unified SAC actor backbone; for DBC, we additionally report the TD3-based variant on HalfCheetah-v5. Detailed experimental settings are provided in Appendix~\ref{appendix:detailed_experimental_setting}, and the complete breakdowns per-actor are reported in Appendix~\ref{appendix:full_experiments}.

In Table~\ref{tab:main:experiments}, all baselines are trained for \(10^{6}\) environment steps and evaluated over three random seeds. As shown, DBC achieves the best overall performance. DSAC exhibits a similar behavior, as it explicitly parameterizes the return distribution with a single Gaussian, limiting its capacity to represent rich, potentially multimodal return distributions. Learning curves of these baselines are demonstrated in Figure~\ref{fig:training_curves}.

\begin{figure*}[htbp]
    \centering
    \includegraphics[width=0.8\linewidth]{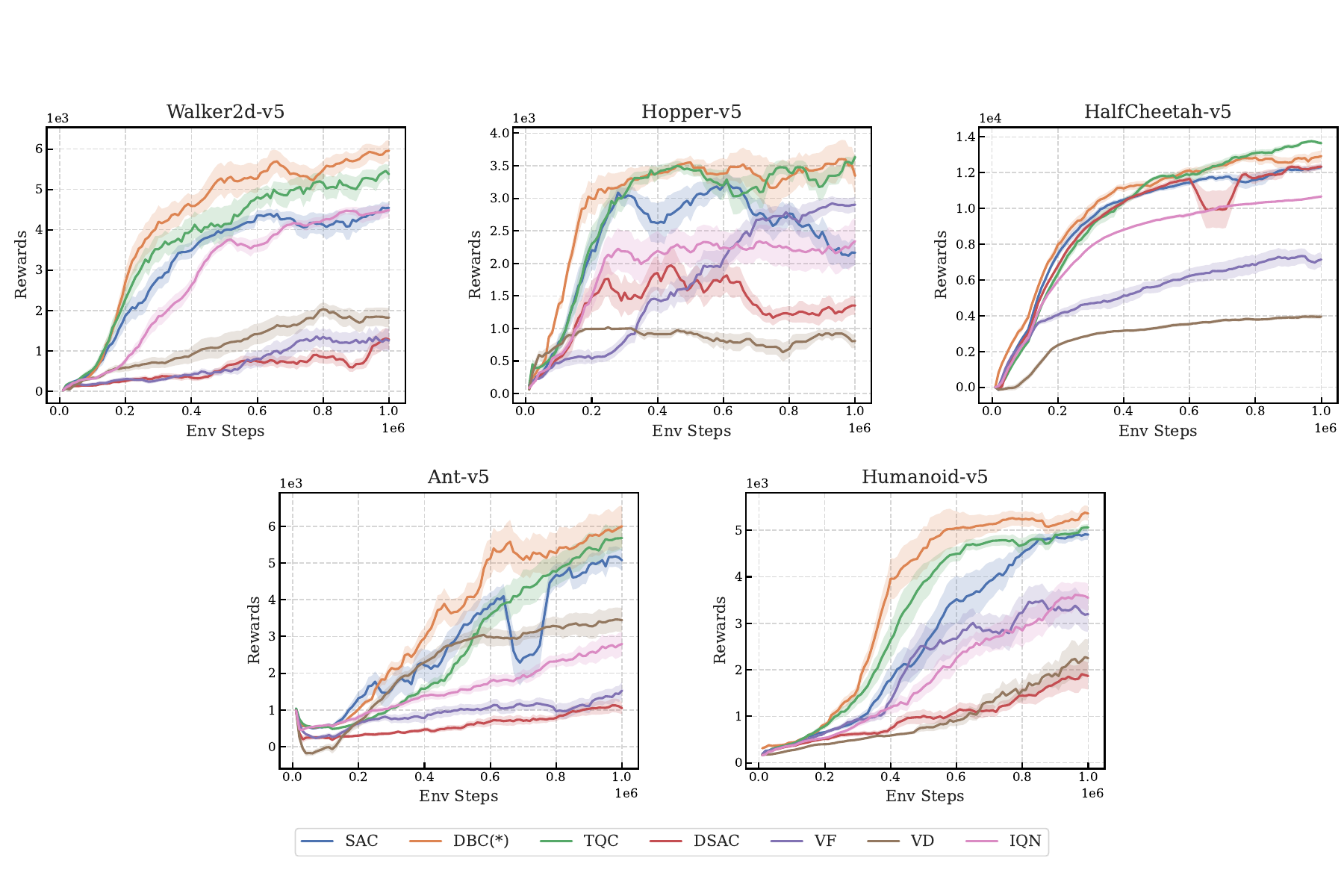}
    % \caption{Learning curves of the proposed method and all baselines on five MuJoCo benchmarks, evaluated every 10k iterations with 10 evaluation episodes. Solid lines denote the mean performance, and shaded regions indicate the standard deviation across three random seeds.}
    \caption{Learning curves of different algorithms on five MuJoCo benchmarks. The x-axis denotes training epochs, and the y-axis denotes episodic return. Curves are smoothed for improved visualization.}
    \label{fig:training_curves}
    \vspace{-4mm}
\end{figure*}

\subsection{DBC as a plug-and-play module.}

To validate the effectiveness of DBC as a plug-and-play critic module, we integrate it into three representative actor algorithms, SAC, TD3, and QVPO,and evaluate the resulting methods on two challenging continuous-control benchmarks, Humanoid-v5 and Walker2d-v5. As shown in Table~\ref{tab:plugin_play}, replacing CDQ with DBC consistently improves performance across all actor backbones. In particular, on Humanoid-v5, DBC yields a maximum improvement of 13\% when paired with the SAC actor, while on Walker2d-v5, it achieves up to a 24\% performance gain when integrated with TD3.

\input{table/plug_in_play}

\subsection{Ablation Study}
\subsubsection{Time Steps}
% In this section, we conduct an ablation study on the number of time steps $T$ to demonstrate that our multi-step critic provides substantial performance gains rather than being a trivial integration of a Diffusion Model and a RL critic. Through this analysis, we aim to uncover the intrinsic merits of the multi-step architecture while identifying potential trade-offs or limitations.

% As illustrated in Table~\ref{tab:ablation:timesteps}, $T=5$ exhibits superior performance compared to the $T=1$ baseline, outperforming it in all tasks except for Hopper-v5. This confirms that the introduction of a multi-step critic significantly enhances the fitting capability of the model. Furthermore, while $T=10$ also leverages multiple steps, the increased iteration count results in an extended backpropagation chain for the actor's updates, which can compromise training stability. Consequently, we recommend $T=5$ as it achieves an optimal trade-off between performance, computational cost, and training stability.
We ablate the number of time steps $M$ to validate that our multi-step critic offers substantial gains beyond a naive integration of Diffusion Models and RL. As shown in Table~\ref{tab:ablation:timesteps}, $M=5$ outperforms $M=1$ in all tasks, confirming that multi-step processing significantly enhances the ability to fit. However, increasing $M$ to 10 introduces instability due to extended backpropagation chains. Therefore, we recommend $M=5$ as the optimal trade-off between performance, cost, and stability.

\input{table/timesteps}

\subsubsection{Weight for Anchor Loss}
Finally, we analyze the effect of the anchor loss introduced in Section~\ref{sec:empirical_anchor_training}. Recall that our critic objective in Eq.~\eqref{eq:loss:total} combines the quantile loss in Eq.~\eqref{eq:loss:quantile_loss} with an anchor term Eq.~\eqref{eq:loss:anchor} weighted by \(w\). Table~\ref{tab:ablation:w2} shows that adding a small anchor weight consistently improves performance: \(w=0.01\) achieves the best results on Ant-v5, Humanoid-v5, and Walker2d-v5, outperforming the \(w=0\) baseline that uses Eq.~\eqref{eq:loss:quantile_loss} alone. In contrast, larger weights (e.g., \(w \ge 0.02)\) can degrade performance, indicating that over-emphasizing the anchor term may distort the training objective. This is because the objective is biased as we have discussed before. Moreover, removing Eq.~\eqref{eq:loss:quantile_loss} leads to a substantial performance drop across all tasks, confirming that the quantile loss is essential for learning the distributional structure while the anchor loss is most effective as a auxiliary signal. Based on these results, \(w=0.01\) is recommended.
\input{table/w2_weight}

% Finally, in this section, we conduct ablation experiments to verify that the Fit Loss introduced in Section~\ref{sec:empirical_anchor_training} indeed facilitates better distribution fitting by the network. Furthermore, given that this introduced loss constitutes a biased optimization objective to some extent, we also investigate the appropriate weighting for this biased term.

% Experimental results in Table~\ref{tab:ablation:w2} demonstrate that a slight introduction of Fit Loss (with $w=0.01$) yields a consistent performance improvement over the baseline where $w=0$ (i.e., using Quantile Loss Eq.~\eqref{eq:loss:quantile_loss} alone). This suggests that Fit Loss effectively facilitates the network in achieving a better distribution fit. However, we observe that starting from $w=0.02$, the performance on certain tasks, such as Ant and Humanoid, begins to underperform the $w=0$ baseline. This degradation becomes more pronounced at $w=0.1$ across multiple tasks including HalfCheetah, Hopper, and Humanoid. These findings indicate that while introducing a biased Fit Loss may incur some performance trade-offs, it remains highly beneficial as it more directly calibrates the model toward the ground-truth distribution. Consequently, incorporating an empirical quantile-based Fit Loss is justified, and we recommend an optimal hyperparameter setting of $w=0.01$.

\section{Conclusion}

In this work, we revisit diffusion-based reinforcement learning from the perspective of value estimation and propose Diffusion Bridge Critics (DBC), a novel distributional RL method that leverages diffusion bridge models to accurately capture the Q-value distribution. By learning the inverse cumulative distribution function of Q-value distribution, DBC avoids the discrete quantile approximation and the Gaussian degradation that arises in vanilla diffusion critics under Bellman backups. Moreover, we derive an integral-consistent discretization to correct the bias induced by finite diffusion steps, leading to more accurate and stable value estimation. As a plug-and-play component, DBC can be seamlessly integrated into existing RL algorithms, and empirical results on MuJoCo benchmarks demonstrate consistent performance gains over prior distributional critics. These findings highlight the importance of expressive and reliable diffusion-based critics, and suggest that improving value estimation is a key direction for advancing diffusion-based RL.

\section*{Acknowledgement}
This work was supported by National Natural Science Foundation of China (62303319, 62406195), Shanghai Local College Capacity Building Program (23010503100), ShanghaiTech AI4S Initiative SHTAI4S202404, HPC Platform of ShanghaiTech University, and MoE Key Laboratory of Intelligent Perception and Human-Machine Collaboration (ShanghaiTech University), Shanghai Engineering Research Center of Intelligent Vision and Imaging. This work was also supported in part by computational resources provided by Fcloud CO., LTD.

\section*{Impact Statement}

This work focuses on improving value estimation in reinforcement learning by introducing diffusion bridge models as critics. By enabling more accurate distributional critic, DBC has the potential to enhance the reliability and efficiency of reinforcement learning systems, particularly in robotics and continuous control applications. 

\bibliography{reference}
\bibliographystyle{icml2026}

%%%%%%%%%%%%%%%%%%%%%%%%%%%%%%%%%%%%%%%%%%%%%%%%%%%%%%%%%%%%%%%%%%%%%%%%%%%%%%%
%%%%%%%%%%%%%%%%%%%%%%%%%%%%%%%%%%%%%%%%%%%%%%%%%%%%%%%%%%%%%%%%%%%%%%%%%%%%%%%
% APPENDIX
%%%%%%%%%%%%%%%%%%%%%%%%%%%%%%%%%%%%%%%%%%%%%%%%%%%%%%%%%%%%%%%%%%%%%%%%%%%%%%%
%%%%%%%%%%%%%%%%%%%%%%%%%%%%%%%%%%%%%%%%%%%%%%%%%%%%%%%%%%%%%%%%%%%%%%%%%%%%%%%
\newpage
\appendix
\onecolumn
\input{appendix}

%%%%%%%%%%%%%%%%%%%%%%%%%%%%%%%%%%%%%%%%%%%%%%%%%%%%%%%%%%%%%%%%%%%%%%%%%%%%%%%
%%%%%%%%%%%%%%%%%%%%%%%%%%%%%%%%%%%%%%%%%%%%%%%%%%%%%%%%%%%%%%%%%%%%%%%%%%%%%%%

\end{document}

%% file: algorithm/train_brief.tex
\begin{algorithm}[htbp]
\caption{Training for DBC}
\label{alg:dbq-training-brief}
\textbf{Input:} replay buffer $\mcD$, online parameters $\theta$, target parameters $\phi$, heads $H$, discount $\gamma$, atoms $K_{\text{target}}$
\begin{algorithmic}[1]
\FOR{$i=1$ \TO\,$N$}
  \STATE Sample $(\boldsymbol{s},\boldsymbol{a},r,\boldsymbol{s}',\mathrm{done})\sim \mcD$, and sample $\boldsymbol{a}'\sim\pi(\cdot\mid \boldsymbol{s}')$
  \STATE Generate target value \(\{z_i\}_{i=1}^{K_{\text{tgt}}}\) by Eq.~\ref{eq:update_rule}
  \STATE Build target sample \(\mathcal{Y}=\{y_i\}_{i=1}^{K_{\mathrm{target}}}\) by Eq.~\eqref{eq:target_particle}.
  \STATE Sample $\tau_i \sim U([0,1])$ and set $z^{\text{start}}(\tau_i) \gets \tau_i$
  \STATE Calculate $\{y^{\mathrm{emp}}_i\}_{i=1}^{K}$ according to Eq.~\eqref{eq:empirical_quantile}
  \FOR{$h=1$ \TO $\,H$}
    \STATE Calculate $\mcL^{(h)}$ according to Eq.~\eqref{eq:loss:total}
  \ENDFOR
  \STATE Update $\theta$ using $\frac{1}{H}\nabla_\theta \sum_{h=1}^H \mcL^{(h)}$
  \STATE Soft update for target network.
\ENDFOR

\end{algorithmic}
\end{algorithm}

%% file: algorithm/sampling.tex
\begin{algorithm}[htbp]
\caption{Integral-consistent ODE Sampling for DBC}
\label{alg:dbc-ode-sampling}
\textbf{Input:} online critic $f_\theta$, state-action $(\boldsymbol{s},\boldsymbol{a})$, quantiles $\{\tau_i\}_{i=1}^{K}$
\textbf{Hyperparameters:} time grid $\{t_m\}_{m=0}^{M}$, interpolation $\xi(\cdot)$
\begin{algorithmic}[1]
\FOR{$i=1$ \TO\,$K$ (In Parallel)}
  \STATE $z \gets z^{\text{start}} \gets \tau_i$
  \FOR{$m=0$ \TO $M-1$}
    \STATE $\hat{z}_{\tau, t_m} \gets f_\theta(z,t_m,\tau_i,\boldsymbol{s},\boldsymbol{a})$
    % \STATE $z \gets z + (\xi(t_m)-\xi(t_{m+1}))(\hat{z}^1-z^{\text{start}})$
    \STATE $z \gets z + \tilde{c}(t_m)(\hat{z}_{\tau, t_m}-z^{\text{start}})$
  \ENDFOR
  \STATE $z^1_{(i)} \gets z$
\ENDFOR
\end{algorithmic}
\end{algorithm}

%% file: table/experiments.tex
\begin{table*}[t]
\caption{Mean evaluation performance of different algorithms in the 5 MuJoCo benchmarks. All experiments are conducted with three random seeds. The returns reported in the tables are the averages of the best-performing results across these seeds.}
\label{tab:main:experiments}
% \vskip 0.15in
\begin{center}
\begin{small}
\begin{tabular}{lccccc}
\toprule
Run & Ant-v5 & HalfCheetah-v5 & Hopper-v5 & Humanoid-v5 & Walker2d-v5 \\
\midrule
SAC  & 6121.8 \(\pm\) 412.7 & 12776.4 \(\pm\) 285.2 & 3630.5 \(\pm\) 50.0  & 5207.1 \(\pm\) 94.1  & 4854.4 \(\pm\) 372.8 \\
IQN  & 3487.4 \(\pm\) 798.5 & 10930.2 \(\pm\) 152.7 & 3301.8 \(\pm\) 360.1 & 4729.2 \(\pm\) 580.5 & 4766.5 \(\pm\) 255.7 \\
TQC  & 6342.6 \(\pm\) 294.7 & \textbf{13996.2 \(\pm\) 191.7} & 3704.6 \(\pm\) 94.5  & 5269.0 \(\pm\) 113.8 & 5802.6 \(\pm\) 234.2 \\
VD   & 3850.4 \(\pm\) 987.1 & 4284.8 \(\pm\) 478.2 & 1052.7 \(\pm\) 20.4  & 4886.9 \(\pm\) 228.8 & 2533.9 \(\pm\) 302.4 \\
VF   & 2650.3 \(\pm\) 163.8 & 8199.5 \(\pm\) 1533.3 & 3310.4 \(\pm\) 24.6  & 4950.3 \(\pm\) 60.3  & 2626.4 \(\pm\) 1387.4 \\
DSAC & 2013.2 \(\pm\) 679.9 & 12974.0 \(\pm\) 191.7 & 3522.9 \(\pm\) 23.5  & 3442.6 \(\pm\) 216.8 & 3297.4 \(\pm\) 1248.3 \\
\midrule
DBC(*) & \textbf{6501.4 \(\pm\) 84.3}  & \textbf{13787.8 \(\pm\) 346.5} & \textbf{3732.5 \(\pm\) 83.2} & \textbf{5906.7 \(\pm\) 198.6} & \textbf{6138.2 \(\pm\) 38.1} \\
\bottomrule
\end{tabular}
\end{small}
\end{center}
\vskip -0.1in
\end{table*}

%% file: table/plug_in_play.tex
% \begin{table*}[htbp]
% \caption{DBC as a Plug-and-Play Critic Module Across Actors}
% \label{tab:dbc:plugin_play}
% \vskip 0.15in
% \begin{center}
% \begin{small}
% \resizebox{\linewidth}{!}{%
% \begin{tabular}{llccccc}
% \toprule
% Actor & Critic & Ant-v5 & HalfCheetah-v5 & Hopper-v5 & Humanoid-v5 & Walker2d-v5 \\
% \midrule
% \multirow{2}{*}{QVPO}
% & CDQ  & 6373.2 \(\pm\) 65.0  & 12103.9 \(\pm\) 422.3 & 3582.0 \(\pm\) 48.1  & 5068.3 \(\pm\) 28.5  & 4986.3 \(\pm\) 279.6 \\
% & DBC(*) & \textbf{6633.8 \(\pm\) 71.2} & 13182.1 \(\pm\) 252.4 & 3721.3 \(\pm\) 116.6 & 5426.3 \(\pm\) 9.0   & 5448.1 \(\pm\) 193.8 \\
% \midrule
% \multirow{2}{*}{SAC}
% & CDQ  & 6349.2 \(\pm\) 562.9 & 12776.4 \(\pm\) 285.2 & 3630.5 \(\pm\) 50.0  & 5207.1 \(\pm\) 94.1  & 4854.4 \(\pm\) 372.8 \\
% & DBC(*) & 6501.4 \(\pm\) 84.3  & 13120.2 \(\pm\) 162.4 & \textbf{3732.5 \(\pm\) 83.2} & \textbf{5906.7 \(\pm\) 198.6} & 6138.2 \(\pm\) 38.1 \\
% \midrule
% \multirow{2}{*}{TD3}
% & CDQ  & 4257.0 \(\pm\) 1711.6 & 11679.0 \(\pm\) 629.8 & 3596.2 \(\pm\) 38.8  & 5067.8 \(\pm\) 33.7  & 5093.7 \(\pm\) 439.8 \\
% & DBC(*) & 5306.0 \(\pm\) 1303.6 & 13787.8 \(\pm\) 346.5 & 3685.7 \(\pm\) 68.2  & 5343.2 \(\pm\) 199.3 & \textbf{6335.5 \(\pm\) 269.4} \\
% \bottomrule
% \end{tabular}
% }
% \end{small}
% \end{center}
% \vskip -0.1in
% \end{table*}
\begin{table}[htbp]
\caption{DBC as a Plug-and-Play Critic Module Across Actors on Humanoid-v5 and Walker2d-v5 benchmarks.}
\label{tab:plugin_play}
\centering
\begin{small}
\begin{tabular}{llcc}
\toprule
Actor & Critic & Humanoid-v5 & Walker2d-v5 \\
\midrule
\multirow{2}{*}{QVPO}
& CDQ    & 5068.3 \(\pm\) 28.5  & 4986.3 \(\pm\) 279.6 \\
& DBC(*) & \textbf{5426.3 \(\pm\) 9.0}  & \textbf{5448.1 \(\pm\) 193.8} \\
\midrule
\multirow{2}{*}{SAC}
& CDQ    & 5207.1 \(\pm\) 94.1  & 4854.4 \(\pm\) 372.8 \\
& DBC(*) & \textbf{5906.7 \(\pm\) 198.6} & \textbf{6138.2 \(\pm\) 38.1} \\
\midrule
\multirow{2}{*}{TD3}
& CDQ    & 5067.8 \(\pm\) 33.7  & 5093.7 \(\pm\) 439.8 \\
& DBC(*) & \textbf{5343.2 \(\pm\) 199.3} & \textbf{6335.5 \(\pm\) 269.4} \\
\bottomrule
\end{tabular}
\end{small}
\end{table}

%% file: table/timesteps.tex
% \begin{table*}[htbp!]
% \caption{Ablation of Flow Steps (T)}
% \label{tab:ablation:timesteps}
% \begin{center}
% \begin{tabular}{lccccc}
% \toprule
% Run & Ant-v5 & HalfCheetah-v5 & Hopper-v5 & Humanoid-v5 & Walker2d-v5 \\
% \midrule
% T=1 & 5150.4 \(\pm\) 2226.3 & 12821.5 \(\pm\) 26.4 & \textbf{3748.1 \(\pm\) 46.8} & 5524.7 \(\pm\) 132.7 & 5809.8 \(\pm\) 308.6 \\
% T=10 & 5946.6 \(\pm\) 379.7 & 13066.6 \(\pm\) 203.0 & 3607.9 \(\pm\) 30.5 & 4848.7 \(\pm\) 88.1 & 6014.3 \(\pm\) 561.6 \\
% T=5 & \textbf{6501.4 \(\pm\) 84.3} & \textbf{13120.2 \(\pm\) 162.4} & 3732.5 \(\pm\) 83.2 & \textbf{5906.7 \(\pm\) 198.6} & \textbf{6138.2 \(\pm\) 38.1} \\
% \bottomrule
% \end{tabular}
% \end{center}
% \end{table*}

\begin{table}[htbp!]
\caption{Ablation of Flow Steps \(M\).}
\label{tab:ablation:timesteps}
\begin{center}
\resizebox{\linewidth}{!}{
\begin{tabular}{lccccc}
\toprule
Run & Ant-v5  & Humanoid-v5 & Walker2d-v5 \\
\midrule
M=1 & 5150.4 \(\pm\) 2226.3  & 5524.7 \(\pm\) 132.7 & 5809.8 \(\pm\) 308.6 \\
M=5 & \textbf{6501.4 \(\pm\) 84.3} & \textbf{5906.7 \(\pm\) 198.6} & \textbf{6138.2 \(\pm\) 38.1} \\
M=10 & 5946.6 \(\pm\) 379.7  & 4848.7 \(\pm\) 88.1 & 6014.3 \(\pm\) 561.6 \\
\bottomrule
\end{tabular}
}
\end{center}
\end{table}

%% file: table/w2_weight.tex
% \begin{table*}[htbp!]
% \caption{Sensitivity Analysis of Loss Weight (w2)}
% \label{tab:ablation:w2}
% \begin{center}
% \begin{tabular}{lccccc}
% \toprule
% Run & Ant-v5 & HalfCheetah-v5 & Hopper-v5 & Humanoid-v5 & Walker2d-v5 \\
% \midrule
% \(w\)=0.0 & 4828.3 \(\pm\) 1439.0 & 12990.1 \(\pm\) 719.7 & 3646.3 \(\pm\) 84.9 & 5807.3 \(\pm\) 175.8 & 5562.9 \(\pm\) 423.2 \\
% \(w\)=0.01 & \textbf{6501.4 \(\pm\) 84.3} & 13120.2 \(\pm\) 162.4 & \textbf{3732.5 \(\pm\) 83.2} & \textbf{5906.7 \(\pm\) 198.6} & \textbf{6138.2 \(\pm\) 38.1} \\
% \(w\)=0.02 & 3947.3 \(\pm\) 2151.8 & \textbf{13165.8 \(\pm\) 301.2} & 3702.6 \(\pm\) 86.5 & 5743.4 \(\pm\) 301.1 & 6097.8 \(\pm\) 375.1 \\
% \(w\)=0.1 & 5716.2 \(\pm\) 741.3 & 12775.5 \(\pm\) 127.8 & 3591.1 \(\pm\) 79.2 & 5589.0 \(\pm\) 50.3 & 5974.6 \(\pm\) 603.2 \\
% \bottomrule
% \end{tabular}
% \end{center}
% \end{table*}

% \begin{table}[htbp!]
% \caption{Ablation of Loss Weight \(w\)}
% \label{tab:ablation:w2}
% \begin{center}
% \resizebox{\linewidth}{!}{
% \begin{tabular}{lccccc}
% \toprule
% Run & Ant-v5  & Humanoid-v5 & Walker2d-v5 \\
% \midrule
% \(w\)=0.0 & 4828.3 \(\pm\) 1439.0 & 5807.3 \(\pm\) 175.8 & 5562.9 \(\pm\) 423.2 \\
% \(w\)=0.01 & \textbf{6501.4 \(\pm\) 84.3} & \textbf{5906.7 \(\pm\) 198.6} & \textbf{6138.2 \(\pm\) 38.1} \\
% \(w\)=0.02 & 3947.3 \(\pm\) 2151.8 & 5743.4 \(\pm\) 301.1 & 6097.8 \(\pm\) 375.1 \\
% \(w\)=0.1 & 5716.2 \(\pm\) 741.3 & 5589.0 \(\pm\) 50.3 & 5974.6 \(\pm\) 603.2 \\
% \bottomrule
% \end{tabular}
% }
% \end{center}
% \end{table}

\begin{table}[htbp]
\caption{Ablation of Loss Weight \(w\).}
\label{tab:ablation:w2}
\begin{center}
\begin{small}
\resizebox{\linewidth}{!}{%
\begin{tabular}{lccc}
\toprule
Run & Ant-v5 & Humanoid-v5 & Walker2d-v5 \\
\midrule
\(w=0.0\) & 4828.3 \(\pm\) 1439.0 & 5807.3 \(\pm\) 175.8 & 5562.9 \(\pm\) 423.2 \\
\(w=0.01\) & \textbf{6501.4 \(\pm\) 84.3} & \textbf{5906.7 \(\pm\) 198.6} & \textbf{6138.2 \(\pm\) 38.1} \\
\(w=0.02\) & 3947.3 \(\pm\) 2151.8 & 5743.4 \(\pm\) 301.1 & 6097.8 \(\pm\) 375.1 \\
\(w=0.1\) & 5716.2 \(\pm\) 741.3 & 5589.0 \(\pm\) 50.3 & 5974.6 \(\pm\) 603.2 \\
No Eq.~\eqref{eq:loss:quantile_loss} & 2517.1 \(\pm\) 2464.7 & 3751.2 \(\pm\) 321.1 & 3859.3 \(\pm\) 1828.7 \\
\bottomrule
\end{tabular}
}
\end{small}
\end{center}
\end{table}

%% file: appendix.tex
\section{Proof}
\label{appendix:proof}
\subsection{Proof of Theorem \ref{thm:gaussian_degradation}}

\begin{theorem}
    \textbf{Gaussian Degradation of Diffusion Critics.} Vanilla Diffusion Critics $f_\theta$ finally degrades into a Gaussian distribution $\mathcal{N}(Q(s,a), \sigma^2)$ with the Bellman update:
    \[
     \mathcal{T}^{\pi} Z(s, a)=R(s, a)+\gamma Z\left(s^{\prime}, a^{\prime}\right), Z\sim f_\theta(s,a).
    \]
    \vspace{-2em}
\end{theorem}
\begin{proof}
Due to the existence of approximation error $\epsilon_t$ in the $t$th Bellman expansion, the target distribution of diffusion is $\hat{Z}(s,a) = R(s,a)+\epsilon_0+\gamma (\hat{Z}(s^\prime,a^\prime)+\epsilon_1)$. Hence, we can derive that 
\[
\hat{Z}(s,a) = Z(s,a) + \sum_{t=0}^\infty \gamma^t\epsilon_t.
\]
Then, for convenience let $Z(s,a) = Q(s,a) + \gamma^{-1}\epsilon_{-1}$ rewrite the $\hat{Z}(s,a) = Z(s,a) + \sum_{t=-1}^\infty\gamma^t \epsilon_t =  Q(s,a) + \sum_{t=-1}^\infty\gamma^t\epsilon_{t}$. Since the approximation error can be assumed to be independent of each other with zero mean, we can apply the Central Limit Theorem: when $\gamma\to 1$, we have 
\[
    \hat{Z}(s,a) \rightarrow \mathcal{N}(Q(s,a), \sigma^2).
\]
In that case, with Bellman backup iteration, the vanilla diffusion critics will finally fall into a Gaussian distribution.
\end{proof}
Notably, the inherent problem here is that diffusion models are designed to model distributions rather than point-to-point mappings. In other words, the objective
\[
    \min_{\theta}\mathbb{E}_{\epsilon, t}\left[\left\|\epsilon-\epsilon_\theta\left(\sqrt{\bar{\alpha}_{t}} x+\sqrt{1-\bar{\alpha}_{t}} \boldsymbol{\epsilon}, t\right)\right\|^2\right], \qquad x = \mathbb{E}(Z),
\]
is not equivalent to
\[
    \min_{\theta}\mathbb{E}_{Z, \epsilon, t}\left[\left\|\epsilon-\epsilon_\theta\left(\sqrt{\bar{\alpha}_{t}} Z+\sqrt{1-\bar{\alpha}_{t}} \boldsymbol{\epsilon}, t\right)\right\|^2\right], \qquad x=Z.
\]
in the diffusion model. Therefore, the approximation error cannot be alleviated simply by collecting more samples (i.e., $\mathbb{E}(Z+\gamma \epsilon_t)=\mathbb{E}(Z)$). 

In contrast, DBC establishes an explicit mapping from quantiles $\tau$ to the Q-value distribution $Z_\tau$, rendering the two objectives equivalent. That corresponds to our quantile loss (\ref{eq:loss:quantile_loss}) and the anchor loss (\ref{eq:loss:anchor}), which both drive the outputs of the diffusion bridge to converge to the corresponding Q-value quantiles $Z_\tau$.

\section{Integral-consistent Discretization}
\label{appendix:integral-consistent}

\subsection{Proof for \eqref{eq:relation_interpolate}}
\label{appendix:integral-consistent:proof}
We formally establish the relationship between the time-dependent velocity coefficient $c(t)$ and the interpolation coefficient $\xi(t)$. Revisiting Eq. \eqref{eq:optimal_solution} from the main text:

\begin{equation} 
\boldsymbol{u}_t^* = g_t \frac{e^{-2\bar{\theta}_{t:1}}}{\bar{\sigma}_{t:1}^2}(x_1 - x_t), \quad x_t = \xi(t) x_0 + (1 - \xi(t)) x_1. \tag{\ref{eq:optimal_solution}} 
\end{equation}
where $\xi(t)$ is a smooth scalar function satisfying the boundary conditions $\xi(1) = 0$ and $\xi(0) = 1$.

To derive the ordinary differential equation (ODE) governing this process, we differentiate Eq. \eqref{eq:optimal_solution} with respect to time $t$:
\begin{equation}
\begin{aligned}
\frac{d{x}_t}{dt} &= \frac{d}{dt} \left[ \xi(t){x}_0 + {x}_1 - \xi(t){x}_1 \right] \\
&= \dot{\xi}(t){x}_0 - \dot{\xi}(t){x}_1 \\
&= -\dot{\xi}(t)({x}_1 - {x}_0).
\end{aligned}
\end{equation}

\subsection{Proof for Theorem~\ref{thm:integral-consistent}}

\begin{proof}
By explicitly expanding the discrete updates, we obtain
\begin{equation}
\begin{aligned}
x_{t_M}
&= x_{t_0} + \left(\int_{t_0}^{t_M} c(t)\,\mathrm{d}t\right)(x_1-x_0) \\
&= x_{t_0} + \bigl(C(t_M)-C(t_0)\bigr)(x_1-x_0) \\
&= x_0 + \bigl(C(1)-C(0)\bigr)(x_1-x_0) \\
&= x_1 ,
\end{aligned}
\end{equation}
where we used the boundary conditions $x_{t_0}=x_0$, $t_0=0$, $t_M=1$, and $C(1)-C(0)=1$.
\end{proof}

\subsection{Schedules for $\theta$}
\label{appendix:integral-consistent:schedule}
The behavior of the GOU bridge is governed by the drift coefficient $\theta(t)$. In our experiments, we evaluated three distinct schedule types. For all schedules, we set the hyperparameters $\theta_{\min}=0.1$ and $\theta_{\max}=5.0$.

\textbf{Constant Schedule}: The drift rate remains invariant throughout the process.
\begin{equation}\theta(t) = 1.0
\end{equation}
\textbf{Linear Schedule}: The drift rate increases linearly, enforcing stronger mean-reversion near the target $t=1$.
\begin{equation}
    \theta(t) = \theta_{\min} + (\theta_{\max} - \theta_{\min}) \cdot t
\end{equation}

\textbf{Cosine Schedule}: A smooth, non-linear schedule that transitions gently at the boundaries.
\begin{equation}
    \theta(t) = \theta_{\min} + \frac{\theta_{\max} - \theta_{\min}}{2} \left( 1 - \cos(\pi t) \right)
\end{equation}

\begin{wraptable}{r}{0.4\textwidth}
    \centering
    \caption{Relative endpoint error ($\%$) of standard Euler discretization under different schedules and $M$.}
    \label{tab:bias_analysis}
    \begin{tabular}{lccc}
        \toprule
        \textbf{Steps ($M$)} & \multicolumn{3}{c}{\textbf{Euler Discretization Error (\%)}} \\
        \cmidrule(lr){2-4}
        & Constant & Linear & Cosine \\
        \midrule
        1    & 14.91 & 21.44 & 21.44 \\
        2    & 9.48  & 6.93  & 18.75 \\
        5    & 4.29  & 5.41  & 6.85  \\
        10   & 2.23  & 3.07  & 3.42  \\
        20   & 1.13  & 1.62  & 1.71  \\
        50   & 0.46  & 0.67  & 0.68  \\
        100  & 0.23  & 0.34  & 0.34  \\
        1000 & 0.02  & 0.03  & 0.03  \\
        \bottomrule
    \end{tabular}
\end{wraptable}

\subsection{Quantitative Analysis of Endpoint Bias}
Standard Euler discretization approximates the continuous integral $\int c(t)dt$ with a discrete sum $\sum c(t_i)\Delta t_i$. Due to the non-linearity of $c(t)$, particularly under dynamic drift schedules, this approximation introduces a systematic discrepancy that prevents the trajectory from reaching the intended boundary. We term this phenomenon \textbf{Endpoint Bias} and quantify it using the relative error $\mathcal{E} = |1 - \sum_{i=0}^{M-1} c(t_i)\Delta t_i| \times 100\%$.

Table~\ref{tab:bias_analysis} presents the errors calculated in various inference steps ($M$). The results reveal that standard Euler discretization incurs substantial bia,  exceeding $21\%$ in low-latency regimes ($M \leq 5$), which can lead to unstable generative semantics and inaccurate reward estimation in online reinforcement learning. In particular, while the Euler method requires a large number of steps ($M \geq 1000$) to suppress bias to negligible levels, our proposed \textbf{Integral-Consistent Discretization} mathematically eliminates this error, maintaining a consistent $\mathcal{E} = 0$ regardless of the number of steps or the complexity of the drift schedule.

% \newpage
\section{Sample Quantile}
\label{appendix:empirical_quantile}
\mbox{This section presents a detailed analysis of the anchor proposed in the main text, focusing on its formal definition and} asymptotic properties. Let $X$ be a real-valued random variable with distribution function
$F \in C^{1}$ and density $f = F'$, satisfying $\mathbb{E}[X^{2}] < \infty$.
Let $X_i$ be i.i.d.\ samples drawn from $F$, and denote the
associated order statistics by
\[
X_{(1)} \le X_{(2)} \le \cdots \le X_{(n)} .
\]

The sample distribution function is defined as
\[
\hat{F}_n(x)
:= \frac{1}{n}\sum_{i=1}^{n}\mathbb{I}(X_i \le x),
\]
where $\mathbb{I}(\cdot)$ denotes the indicator function.
For a fixed quantile level $\tau \in (0,1)$, the (generalized) empirical
$\tau$-quantile is defined via the generalized inverse
\[
\hat{F}_n^{-1}(\tau)
:= \inf\bigl\{ x \in \mathbb{R} : \hat{F}_n(x) \ge \tau \bigr\}.
\]
It follows immediately that
\[
\hat{F}_n^{-1}(\tau) = X_{(\lceil n\tau \rceil)}.
\]
For notational convenience, we write $\hat{X}_{\tau} := \hat{F}_n^{-1}(\tau)$.

\paragraph{Sample quantiles as a empirical minimizer of the quantile loss.}
Define the quantile loss $\rho_\tau:\mathbb{R}\to\mathbb{R}$ by
\[
\rho_\tau(u)=u\bigl(\tau-\mathbb{I}(u<0)\bigr),
\qquad \forall u\in\mathbb{R}.
\]
Consider the sample risk as a function of a scalar decision variable $\theta\in\mathbb{R}$:
\[
L(\theta):=\sum_{i=1}^{n}\rho_\tau(X_i-\theta).
\]
Then every minimizer of $L(\theta)$ is an sample $\tau$-quantile.

\begin{lemma}[Sample quantile as a minimizer of the quantile loss]
Let $\tau\in(0,1)$ and define
\[
n_{<}(\theta):=\sum_{i=1}^{n}\mathbb{I}(X_i<\theta),
\qquad
n_{\le}(\theta):=\sum_{i=1}^{n}\mathbb{I}(X_i\le \theta).
\]
Then the set of minimizers of $L(\theta)$ is
\[
\operatorname*{arg\,min}_{\theta\in\mathbb{R}} L(\theta)
=
\left\{\theta\in\mathbb{R}:\; n_{<}(\theta)\le n\tau \le n_{\le}(\theta)\right\}.
\]
In particular, $\hat{X}_\tau=X_{(\lceil n\tau\rceil)}$ is always a minimizer.
Moreover, if $X_{(\lceil n\tau\rceil)}<X_{(\lceil n\tau\rceil+1)}$ (i.e., there is no tie at the $\lceil n\tau\rceil$-th order statistic), then the minimizer is unique and equals $\hat{X}_\tau$.
\end{lemma}

\begin{proof}
For each $i$, define $\ell_i(\theta)=\rho_\tau(X_i-\theta)$.
As a function of $\theta$, $\ell_i$ is convex and piecewise linear.
Its subgradient is
\[
\partial \ell_i(\theta)=
\begin{cases}
\{-\tau\}, & X_i>\theta,\\
\{1-\tau\}, & X_i<\theta,\\
[-\tau,\,1-\tau], & X_i=\theta,
\end{cases}
\]
where $\partial$ denotes the subdifferential (a set) in $\mathbb{R}$.
By subdifferential calculus for sums of convex functions,
\[
\partial L(\theta)=\sum_{i=1}^{n}\partial \ell_i(\theta).
\]
Let $n_{<}(\theta)$ and $n_{\le}(\theta)$ be defined as in the statement.
Then the indices with $X_i<\theta$ contribute $1-\tau$, those with $X_i>\theta$ contribute $-\tau$, and those with $X_i=\theta$ contribute an interval $[-\tau,1-\tau]$.
Collecting terms yields
\[
\partial L(\theta)=\bigl[n_{<}(\theta)-n\tau,\; n_{\le}(\theta)-n\tau\bigr].
\]
A point $\theta$ minimizes the convex function $L$ if and only if $0\in \partial L(\theta)$, which is equivalent to
\[
n_{<}(\theta)\le n\tau \le n_{\le}(\theta).
\]
This proves the characterization of $\operatorname*{arg\,min} L(\theta)$.
乡村Since $\hat{X}_\tau=\hat{F}_n^{-1}(\tau)=X_{(\lceil n\tau\rceil)}$ satisfies the above inequality, it is a minimizer.
If additionally $X_{(\lceil n\tau\rceil)}<X_{(\lceil n\tau\rceil+1)}$, then the inequality pins down a single point, hence the minimizer is unique and equals $\hat{X}_\tau$.
\end{proof}

In addition, the sample quantile is guaranteed to converge to the true quantile as \(n\) approaches infinity. It is guaranteed by the following theorem.

\begin{theorem}[Consistency of the sample quantile]
\label{thm:quantile_consistency}
Let $X$ be a real-valued random variable with cumulative distribution function $F$, and let
\[
X_\tau := F^{-1}(\tau) = \inf\{x\in\mathbb{R} : F(x) \ge \tau\},
\qquad \tau \in (0,1),
\]
denote the population $\tau$-quantile. Let $\hat{F}_n$ be the sample distribution function based on i.i.d.\ samples $\{X_i\}_{i=1}^n$, and define the sample $\tau$-quantile by
\[
\hat{X}_\tau := \hat{F}_n^{-1}(\tau).
\]
Assume that $X_\tau$ is unique, i.e., $F$ is strictly increasing in a neighborhood of $x_\tau$, and that $F$ is continuous at $X_\tau$.Then
\[
\hat{X}_\tau \xrightarrow{\text{a.s.}} X_\tau \quad \text{as } n \to \infty.
\]
\end{theorem}

\begin{proof}
By the Glivenko--Cantelli theorem, the sample distribution function $\hat{F}_n$ converges uniformly to $F$ almost surely. Since $F$ is continuous and strictly increasing in a neighborhood of $x_\tau$, the generalized inverse mapping $F^{-1}$ is continuous at $\tau$.The almost sure convergence $\hat{X}_\tau \to X_\tau$ then follows from the continuous mapping theorem. A detailed proof can be found in \citep[p.~266]{van2000asymptotic}.
\end{proof}

\paragraph{Asymptotic Normality of the Sample Quantile.} We now establish the asymptotic distribution of the empirical quantile $\hat{X}_\tau$.

\begin{theorem}[Asymptotic normality of the sample quantile]
\label{thm:quantile_asymptotic_normality}
Let $X_1, X_2, \dots, X_n$ be i.i.d. random variables with cumulative distribution function $F$ and probability density function $f$. Fix $\tau \in (0, 1)$ and denote the population $\tau$-quantile by $X_\tau = F^{-1}(\tau)$ and the empirical $\tau$-quantile by $\hat{X}_\tau = \hat{F}_n^{-1}(\tau)$. If $f$ is continuous at $X_\tau$ and $f(X_\tau) > 0$, then
\begin{equation}
\sqrt{n}\,(\hat{X}_\tau - X_\tau) \overset{d}{\longrightarrow} \mathcal{N}\!\left(0, \frac{\tau(1-\tau)}{f^2(X_\tau)}\right), \quad n\to\infty .
\end{equation}
\end{theorem}

\begin{proof}
For any $t\in\mathbb{R}$, the event $\{\sqrt{n}(\hat{X}_\tau - X_\tau)\leq t\}$ is equivalent to $\hat{F}_n(X_\tau + t/\sqrt{n})\geq \tau$. Hence
\begin{equation}
P\!\left(\sqrt{n}(\hat{X}_\tau - X_\tau)\leq t\right) = P\!\left(\hat{F}_n(X_\tau + t/\sqrt{n})\geq \tau\right).
\end{equation}

Introduce the decomposition
\begin{equation}
\sqrt{n}\bigl[\hat{F}_n(X_\tau + t/\sqrt{n}) - \tau\bigr] = A_n + B_n,
\end{equation}
where
\begin{equation}
\begin{aligned}
A_n &= \sqrt{n}\Bigl[\hat{F}_n(X_\tau + t/\sqrt{n}) - F(X_\tau + t/\sqrt{n})\Bigr],\\[2mm]
B_n &= \sqrt{n}\Bigl[F(X_\tau + t/\sqrt{n}) - \tau\Bigr].
\end{aligned}
\end{equation}

We analyze $A_n$ and $B_n$ separately.

For $B_n$, we expand $F$ around $X_\tau$ and using $F(X_\tau)=\tau$ and obtain
\begin{equation}
B_n = f(X_\tau) t + o(1), \qquad n\to\infty .
\end{equation}

For $A_n$, we define $Y_{n,i} = \mathbb{I}(X_i \leq X_\tau + t/\sqrt{n})$ for $i=1,\dots,n$. Then $\{Y_{n,i}\}$ are i.i.d. Bernoulli random variables with
\begin{equation}
\begin{aligned}
\mathbb{E}[Y_{n,i}] &= F(X_\tau + t/\sqrt{n}), \\
\operatorname{Var}(Y_{n,i}) &= F(X_\tau + t/\sqrt{n})\bigl[1-F(X_\tau + t/\sqrt{n})\bigr].
\end{aligned}
\end{equation}
It is straightforward to verify that
\begin{equation}
\lim_{n\to\infty} \mathbb{E}[Y_{n,i}] = \tau, \qquad
\lim_{n\to\infty} \operatorname{Var}(Y_{n,i}) = \tau(1-\tau).
\end{equation}
Applying the Lindeberg--Feller central limit theorem yields
\begin{equation}
A_n = \sqrt{n}\Bigl(\frac{1}{n}\sum_{i=1}^n Y_{n,i} - F(X_\tau + t/\sqrt{n})\Bigr) \overset{d}{\longrightarrow} \mathcal{N}\bigl(0,\tau(1-\tau)\bigr).
\end{equation}

Noe, combining the two terms and according to Slutsky's theorem, we have that
\begin{equation}
A_n + B_n \overset{d}{\longrightarrow} \mathcal{N}\bigl(f(X_\tau)t,\;\tau(1-\tau)\bigr).
\end{equation}
Consequently,
\begin{equation}
\begin{aligned}
\lim_{n\to\infty} P\!\left(\sqrt{n}(\hat{X}_\tau - X_\tau)\leq t\right)
&= \lim_{n\to\infty} P\!\left(A_n + B_n \geq 0\right) \\[1mm]
&= \Phi\!\left(\frac{f(X_\tau)t}{\sqrt{\tau(1-\tau)}}\right),
\end{aligned}
\end{equation}
where $\Phi$ denotes the standard normal distribution function. This shows that the limiting distribution of $\sqrt{n}(\hat{X}_\tau - X_\tau)$ is $\mathcal{N}\!\bigl(0,\tau(1-\tau)/f^2(X_\tau)\bigr)$, completing the proof.
\end{proof}

\paragraph{Bias of the Sample Quantile.} Before presenting the formal result, we remark that no universally unbiased estimator exists for quantiles; we can only provide asymptotic statements. For a detailed discussion of this point, see~\citep[p.~84]{lehmann1998theory}.

\begin{theorem}[Asymptotic bias of the sample quantile]
\label{thm:bias_sample_quantile}
Let $X_1, X_2, \dots, X_n$ be i.i.d. random variables with cumulative distribution function $F$ and quantile function $Q(\tau)=F^{-1}(\tau)$. Fix $\delta>0$ and consider $\tau\in[\delta,1-\delta]$. Assume that:
\begin{enumerate}[label=(\roman*)]
    \item The density $f$ satisfies $\inf_{\tau\in[\delta,1-\delta]} f(Q(\tau)) \ge c > 0$ for some constant $c$;
    \item $Q$ is differentiable on $[\delta,1-\delta]$;
    \item $\mathbb{E}[X^2]<\infty$.
\end{enumerate}
Then the bias of the empirical quantile $\hat{Q}_n(\tau)=\hat{X}_\tau$ satisfies
\begin{equation}
\bigl|\mathbb{E}[\hat{Q}_n(\tau)] - Q(\tau)\bigr| = o\!\left(\frac{1}{\sqrt{n}}\right).
\end{equation}
\end{theorem}

\begin{proof}
Write $U_i = F(X_i) \sim \operatorname{Uniform}(0,1)$ and let $U_{(1)}\le\cdots\le U_{(n)}$ be the order statistics. Set $k = \lceil n\tau\rceil$ and $V_n = U_{(k)}-\tau$. Then $\hat{Q}_n(\tau) = Q(U_{(k)}) = Q(\tau+V_n)$.

A first-order Taylor expansion yields
\begin{equation}
\hat{Q}_n(\tau) - Q(\tau) = Q'(\tau) V_n + r(V_n),
\label{eq:bias_taylor}
\end{equation}
where the remainder $r(v)$ satisfies the Peano form: for every $\epsilon>0$ there exists $\delta_\epsilon>0$ such that $|r(v)|\le\epsilon|v|$ whenever $|v|\le\delta_\epsilon$.

From properties of the Beta distribution,
\begin{equation}
\mathbb{E}[V_n] = \frac{k}{n+1} - \tau, \qquad
\operatorname{Var}(V_n) = \frac{\tau(1-\tau)}{n} + O\!\left(\frac{1}{n^2}\right).
\end{equation}
By Jensen's inequality,
\begin{equation}
\mathbb{E}\bigl[|V_n|\bigr] \le \sqrt{\mathbb{E}[V_n^2]} 
= \sqrt{\operatorname{Var}(V_n) + \bigl(\mathbb{E}[V_n]\bigr)^2}
= O\!\left(\frac{1}{\sqrt{n}}\right).
\label{eq:Vn_moment}
\end{equation}

To bound $\mathbb{E}[|r(V_n)|]$, we first estimate the second moment of $r(V_n)$. Using the elementary inequality $(a+b+c)^2 \le 3(a^2+b^2+c^2)$,
\begin{equation}
\begin{aligned}
r(V_n)^2 
&= \bigl(Q(\tau+V_n) - Q(\tau) - Q'(\tau)V_n\bigr)^2 \\
&\le 3\Bigl(Q^2(\tau+V_n) + Q^2(\tau) + (Q'(\tau))^2 V_n^2\Bigr).
\end{aligned}
\end{equation}
Hence
\begin{equation}
\mathbb{E}\bigl[r(V_n)^2\bigr] \le 3\Bigl(\mathbb{E}\bigl[X_{(k)}^2\bigr] + Q^2(\tau) + (Q'(\tau))^2\mathbb{E}\bigl[V_n^2\bigr]\Bigr),
\end{equation}
where $X_{(k)}=Q(U_{(k)})$ is the $k$-th order statistic of the original sample. Because $\mathbb{E}[X^2]<\infty$, there exists a constant $A>0$ such that $\mathbb{E}[r(V_n)^2] \le A n$.

Next, the Dvoretzky--Kiefer--Wolfowitz inequality gives
\begin{equation}
P\bigl(|V_n|\ge t\bigr) \le 2\exp(-2nt^2), \qquad t>0.
\label{eq:DKW_bound}
\end{equation}

Split the expectation of $|r(V_n)|$ into two parts. For a given $\epsilon>0$, take the corresponding $\delta_\epsilon$:
\begin{equation}
\begin{aligned}
\mathbb{E}\bigl[|r(V_n)|\bigr]
&= \mathbb{E}\Bigl[|r(V_n)|\,\mathbb{I}\{|V_n|\le\delta_\epsilon\}\Bigr]
   + \mathbb{E}\Bigl[|r(V_n)|\,\mathbb{I}\{|V_n|>\delta_\epsilon\}\Bigr] \\[1mm]
&\le \epsilon\,\mathbb{E}\bigl[|V_n|\bigr]
   + \sqrt{\mathbb{E}\bigl[r(V_n)^2\bigr]}\,\sqrt{P\bigl(|V_n|>\delta_\epsilon\bigr)} \\[1mm]
&\le \epsilon\,O\!\left(\frac{1}{\sqrt{n}}\right)
   + \sqrt{A n}\,\sqrt{2\exp(-2n\delta_\epsilon^2)} \\[1mm]
&= \epsilon\,O\!\left(\frac{1}{\sqrt{n}}\right) + o\!\left(\frac{1}{\sqrt{n}}\right).
\end{aligned}
\end{equation}
Since $\epsilon>0$ is arbitrary, we obtain
\begin{equation}
\mathbb{E}\bigl[|r(V_n)|\bigr] = o\!\left(\frac{1}{\sqrt{n}}\right).
\label{eq:remainder_bound}
\end{equation}

Finally, combining \eqref{eq:bias_taylor}, \eqref{eq:Vn_moment} and \eqref{eq:remainder_bound},
\begin{equation}
\begin{aligned}
\bigl|\mathbb{E}[\hat{Q}_n(\tau)] - Q(\tau)\bigr|
&\le |Q'(\tau)|\,\bigl|\mathbb{E}[V_n]\bigr| + \mathbb{E}\bigl[|r(V_n)|\bigr] \\[1mm]
&\le \frac{1}{c}\cdot\frac{1}{n+1} + o\!\left(\frac{1}{\sqrt{n}}\right) \\[1mm]
&= o\!\left(\frac{1}{\sqrt{n}}\right),
\end{aligned}
\end{equation}
which completes the proof.
\end{proof}

Accordingly, letting $n = K_{\mathrm{tgt}}$, we show that, given
$K_{\mathrm{tgt}}$ target samples, the resulting asymptotic bias is
$o\!\left(1/\sqrt{K_{\mathrm{tgt}}}\right)$.

\noindent\textbf{Remark.} The assumption $\inf_{\tau\in[\delta,1-\delta]} f(Q(\tau)) \ge c>0$ guarantees that $Q'(\tau)=1/f(Q(\tau))$ is bounded on $[\delta,1-\delta]$. If one further assumes that $Q'$ satisfies Lipschitz Condition, a sharper $O(1/n)$ bias bound can be derived. Restricting to the inner interval $[\delta,1-\delta]$ covers the quantile levels that are most relevant in practice while providing the necessary regularity for the analysis.

\newpage
\section{Full Algorithm}
\label{appendix:full_algorithm}
Here, the \textbf{\textit{generator}} corresponds to a specific instantiation of Eq.~\eqref{eq:unidb_soc_ode} under a predefined scheduling scheme.

\input{algorithm/train}

\section{Baseline Methods and Implementation Details}
\label{appendix:baselines}

Here we provide additional details on the baseline methods reported in Table~\ref{tab:main:experiments} of the main paper.
We summarize their distributional assumptions and generative mechanisms, and clarify the implementation and comparison protocols adopted in our experiments.

\paragraph{TD3~\cite{fujimoto2018addressing}.}
CDQ (Clipped Double Q-Learning), originally introduced in TD3, models the critic as a scalar-valued function and mitigates overestimation by taking the minimum over multiple Q-value estimates.
It does not explicitly maintain a return distribution and thus corresponds to a non-distributional critic with an implicit point-mass assumption. The official torch implementation for TD3 is available at:
\[
\texttt{https://github.com/sfujim/TD3}
\]

\paragraph{SAC~\cite{haarnoja2018soft}.}
SAC (Soft Actor-Critic) is an off-policy actor--critic algorithm that learns a stochastic policy under the maximum-entropy reinforcement learning framework. By explicitly encouraging policy entropy during training, SAC improves exploration and robustness while maintaining strong sample efficiency. In practice, SAC maintains a pair of Q-functions and uses a clipped double-Q target to mitigate overestimation. It also supports automatic temperature adjustment, which adapts the strength of the entropy regularization to match a target entropy. We implement SAC in our unified codebase and use it as the default actor backbone in our experiments. The official tensorflow implementation for SAC is available at:
\[
\texttt{https://github.com/haarnoja/sac}
\]

\paragraph{QVPO~\cite{ding2024diffusion}.}
QVPO (Q-weighted Variational Policy Optimization) is an online diffusion-policy reinforcement learning algorithm that parameterizes the actor as a conditional diffusion model, thereby capturing multimodal action distributions. To bridge diffusion-model training with online policy optimization, QVPO introduces a Q-weighted variational objective that re-weights the diffusion training signal using value estimates and further employs a weight transformation to handle general settings where raw Q-values may not be strictly positive. In addition, QVPO incorporates a tailored entropy regularization to encourage exploration and proposes an efficient behavior/action selection strategy to reduce the variance of diffusion-policy rollouts during online interaction. We follow the official implementation and integrate it into our unified framework:
\[
\texttt{https://github.com/wadx2019/qvpo}.
\]

\paragraph{DSAC~\cite{dsac}.}
DSAC models the return distribution as a Gaussian distribution, parameterized by a learned mean and variance.
Under this assumption, the entire return distribution can be represented compactly, and target sampling reduces to sampling from a Gaussian.
We base our implementation on the official DSAC codebase:
\[
\texttt{https://github.com/Jingliang-Duan/DSAC-v2}.
\]
Specifically, we integrate the \texttt{DSAC-v1} implementation into our unified code framework and use it as the DSAC baseline.

\paragraph{IQN~\cite{dabney2018implicit}.}
IQN explicitly models the continuous quantile function of the return distribution.
Quantile levels are embedded using cosine features, which improves the representation of high-frequency distributional structures.
We implement IQN following the original paper.
When combined with a SAC actor, IQN can be viewed as a particular instantiation of the DSAC~\cite{ma2025dsac} framework, where DSAC here denotes the general distributional SAC formulation rather than the Gaussian-based DSAC baseline discussed earlier.

\paragraph{TQC~\cite{kuznetsov2020controlling}.}
TQC extends quantile-based critics by aggregating multiple critic heads and truncating the highest quantile estimates (Drop-Top) to control overestimation.
Compared to CDQ, which may introduce systematic underestimation by taking a hard minimum, TQC provides a finer-grained and more robust control mechanism.
Our implementation is adapted from the official PyTorch reference:

\[
\texttt{https://github.com/alxlampe/tqc\_pytorch}.
\]

and integrated into our unified framework.

\paragraph{VF~\cite{dong2025value}.}
Value Flows (VF) combine Flow Matching with distributional critics and primarily target offline reinforcement learning.
The original implementation is written in JAX and available at:
\[
\texttt{https://github.com/chongyi-zheng/value-flows}.
\]
We re-implement the method in PyTorch based on the official code and extract only the critic component for comparison, decoupling it from the original actor design.

\paragraph{VDRL~\cite{hu2025value}.}
VDRL (Value Diffusion Reinforcement Learning) models the return distribution using diffusion models and is designed for online reinforcement learning.
At the time of submission, no official implementation is publicly available.
We implement VDRL from scratch based on the descriptions in the original paper and integrate it into our unified framework as a baseline.

\paragraph{Implementation and comparison protocol.}
To ensure a fair and consistent comparison across baselines, we adopt the following unified design principles:

\begin{itemize}
  \item \textbf{Implementation source.}
  For methods with official implementations (e.g., DSAC, TQC, VF), we base our implementations on the official code.
  For methods originally implemented in JAX (e.g., VF), we provide corresponding PyTorch re-implementations.
  All baselines are integrated into the same unified code framework.

  \item \textbf{Actor--critic decoupling.}
  For all baselines, we decouple the critic from the actor and evaluate only the critic component under a unified training and evaluation pipeline.

  \item \textbf{Target construction.}
  For entropy-regularized methods, target return samples are constructed as
  \[
  r + \gamma \bigl(Z^{\mathrm{tgt}}(s',a') - \alpha \log \pi(a' \mid s')\bigr).
  \]
  For non-entropy-based methods, target return samples are constructed as
  \[
  r + \gamma Z^{\mathrm{tgt}}(s',a').
  \]

  \item \textbf{Distributional aggregation.}
  For critics with multiple heads or multiple return distributions, we consider the following aggregation schemes:
  \begin{enumerate}
    \item Mean aggregation:
    \[
    \frac{1}{H} \sum_{h=1}^{H} \mathbb{E}[Z_h].
    \]
    \item Min aggregation:
    \[
    \mathbb{E}\bigl[\min_{h} Z_h\bigr].
    \]
    \item Drop-Top aggregation (TQC):
    truncating the largest quantile estimates before computing expectations.
  \end{enumerate}
\end{itemize}

These design choices ensure that all baselines are evaluated under consistent training signals and target semantics, isolating the effects of different distributional assumptions and generative mechanisms. We provide a concise summary of these methods in Table~\ref{tab:baseline-paradigms} to facilitate a clear comparison and highlight the key differences among them.

\input{table/critic_type}

\newpage
\section{Detailed Experimental Settings}
\label{appendix:detailed_experimental_setting}

In this section, we briefly describe the hyperparameters used for each baseline in our experiments. Since our codebase adopts a fully decoupled design for the actor and the critic, we present their hyperparameters separately. Specifically, we report the settings for the actor algorithms, QVPO, SAC, and TD3, as well as for the critic algorithms, DSAC, IQN, DBC, TQC, CDQ, VF, and VD.
A detailed summary of the actor hyperparameters is provided in Table~\ref{tab:actor_hyperparams}.

\input{table/actor_hyperparameter}

Due to the nature of distributional reinforcement learning, the sampling and aggregation schemes of distributional critics are an important part of the experimental configuration and thus are reported separately in Table~\ref{tab:distribution_hyperparameters}. In this table, we only specify the aggregation mode for combining the outputs of heads; by default, the final critic output is obtained by averaging over samples (i.e. $\frac{1}{n}\sum_{i=1}^{n} z_i$). DSAC is excluded from Table~\ref{tab:distribution_hyperparameters} because it models the return distribution as a Gaussian and directly predicts the return mean (and variance), rather than relying on multi-sample estimation and aggregation as in other distributional methods.

\input{table/critic_hyperparameter}

\input{table/distributional_hyperparameter}

Table~\ref{tab:drop_ratio} presents the configurations for DBC and TQC, which are the two algorithms adopting the DropTop aggregation scheme. Notably, our TQC's setting strictly adheres to the settings reported in its original paper.

\input{table/drop_ratio}

\section{Full Experiments}
\label{appendix:full_experiments}

\paragraph{Experimental Environment.}
All experiments were conducted on a Linux server running Ubuntu 22.04.5 LTS with kernel version 5.15.0. The machine is equipped with two Intel Xeon Platinum 8480+ CPUs, providing 224 logical cores in total and 2~TiB of system memory. For acceleration, we use 8 NVIDIA H20 GPUs, each with 96~GB of HBM memory, driven by NVIDIA driver 580.95.05. The software stack consists of Python~3.12.12 and PyTorch~2.9.1 compiled with CUDA~12.8, with CUDA runtime support verified during execution. Unless otherwise stated, all reported results are obtained under this hardware and software configuration.

Throughout this section, unless stated otherwise, we adopt the following default hyperparameters: \(M = 5\), \(w = 0.01\), and \(\theta_t = 1\).
\subsection{Detailed Results for All Baselines}
\label{appendix:full_experiments:all_baselines}

\input{table/full_experiments}

\subsection{Ablation for Different Schedules}
\label{appendix:full_experiments:schedules}

We observe that the constant and linear schedules exhibit comparable performance, whereas the cosine schedule results in substantially higher variance. Therefore, for the case of \(M=5\), we recommend using either the constant or the linear schedule.

\input{table/schedule}

\subsection{Ablation for Different Weights}
\label{appendix:full_experiments:weights}

We further investigate the effect of different values of the weight \(w\) across a wider range of tasks. Empirical results show that a small weight \(w\) (e.g., 0.01 or 0.02) consistently leads to performance improvements. However, as \(w\) increases, the performance may begin to deteriorate. For reference, when training is performed solely with the Anchor Loss, the resulting performance is generally poor. This indicates that the Anchor Loss should not be used in isolation, but rather as a weak auxiliary loss to facilitate convergence. Based on these observations, we recommend setting \(w = 0.01\).

\input{table/w2_weight_full}

\subsection{Ablation for Different Time Steps}
\label{appendix:full_experiments:timesteps}
We further examine the impact of the number of flow steps (M) across a wider range of tasks. Consistent with the results in the main text, we find that (M = 5) yields the best overall performance.

\input{table/timesteps_full}

\subsection{Training Time Across Benchmarks}
\label{appendix:full_experiments:training_time}
As shown in Table~\ref{tab:training_time}, although DBC employs a multi-step reward generation process in the critic, the resulting training-time overhead remains manageable. Compared to other baselines, DBC incurs only an approximately \(2×\) increase in training time.
\input{table/training_times}

%% file: algorithm/train.tex
\begin{algorithm}[htbp]
\caption{GOU-BQC: Training and Inference}
\label{alg:dbc-train-infer-goub}
\begin{algorithmic}[1]
\STATE \textbf{Inputs:} replay buffer $\mcD$; online parameters $\theta$; target parameters $\phi$
\STATE \textbf{Hyperparameters:} heads $H$; discount $\gamma$; temperature $\alpha$; target atoms $K_{\text{target}}$; online atoms $K$; time steps $T$; TQC drop $d$; max updates $N_{\text{step}}$; EMA coefficient $\tau_{\text{tgt}}$; Generator
\STATE \textbf{Notation:} $m\in\{0,1\}$ is the terminal mask; $\textsf{sac}=\alpha\log\pi(\boldsymbol{a}'\mid\boldsymbol{s}')$

\vspace{-0.4em}\rule{\linewidth}{0.4pt}\vspace{-0.2em}

\STATE \textbf{Part I: Training (performed for $N_{\text{step}}$ updates)}
\STATE $i \gets 0$
\WHILE{$i < N_{\text{step}}$}
  \STATE $i \gets i + 1$
  \STATE Sample $(\boldsymbol{s},\boldsymbol{a}, r, \boldsymbol{s}', m)\sim\mcD$
  % \STATE Sample $\boldsymbol{a}'\sim\pi(\cdot\mid\boldsymbol{s}')$ and obtain $\log\pi(\boldsymbol{a}'\mid\boldsymbol{s}')$
  % \STATE Set $\textsf{sac} \gets \alpha \log\pi(\boldsymbol{a}'\mid\boldsymbol{s}')$

  \STATE \textbf{(A) Target atoms (no gradient).}
  \STATE Sample $\tau' \sim U([0,1]^{K_{\text{target}}})$ and set ${z'}^{\text{start}} \gets \tau'$
  \STATE $\{z^{(h)}_{\text{target}}\}_{h=1}^H \gets \text{Generator}({z'}^{\text{start}},\boldsymbol{s}',\boldsymbol{a}',\tau', \phi)$
  \STATE Stack atoms $Z \gets \text{Concat}(z^{(1)}_{\text{target}},\dots,z^{(H)}_{\text{target}})$
  \STATE Apply truncation $Z_{\text{tqc}} \gets \text{DropTop}(Z, d\cdot H)$
  \STATE Build target samples $y_{\text{target}} \gets r + m\gamma\,(Z_{\text{tqc}}-\textsf{sac})$

  \STATE \textbf{(B) Online samples and empirical quantiles.}
  \STATE Sample $\tau \sim U([0,1]^K)$ and set $z^{\text{start}}\gets \tau$
  \STATE $\hat{z}_{\tau} \gets \text{SampleQuantile}(y_{\text{target}}, \tau)$

  \STATE \textbf{(C) Online predictions at $t=0$ and $t=t_m$.}
  \STATE Sample $t_m\sim \text{Unif}(0,1)$
  \STATE $z_{t_m} \gets \xi(t_m)\,\hat{z}^{\text{end}} + (1-\xi(t_m))\,z^{\text{start}}$

  \FOR{$h=1$ \TO $H$}
    \STATE $\hat{z}^{(h)}_{\tau,0} \gets f_{\theta,h}(z^{\text{start}},\boldsymbol{s},\boldsymbol{a},\tau,0)$
    \STATE $\hat{z}^{(h)}_{\tau,m} \gets f_{\theta,h}(z_{t_m},\boldsymbol{s},\boldsymbol{a},\tau,t_m)$
    \STATE $\mcL^{(h)} \gets \mcL\!\left(\hat{z}^{(h)}_{\tau, 0}, \hat{z}^{(h)}_{\tau, m}; \hat{z}_{\tau}, z^{\text{start}}, t_m\right)$ by Eq.~\eqref{eq:loss:total}.
  \ENDFOR

  \STATE \textbf{(D) Loss and updates.}
  \STATE $\mcL_{\text{total}} \gets \frac{1}{H}\sum_{h=1}^H \mcL^{(h)}$
  \STATE Update $\theta$ by descending $\nabla_{\theta}\mcL_{\text{total}}$
  \STATE Soft-update target parameters $\phi \leftarrow \tau_{\text{tgt}}\theta + (1-\tau_{\text{tgt}})\phi$
\ENDWHILE

\vspace{-0.2em}\rule{\linewidth}{0.4pt}\vspace{-0.2em}

\STATE \textbf{Part II: Inference (Q-value evaluation)}
\STATE \textbf{Input:} a state-action pair $(\boldsymbol{s},\boldsymbol{a})$
\STATE Sample $\tau \sim U([0,1]^K)$ and set $z^{\text{start}} \gets \tau$
\STATE $\{z^{(h)}\}_{h=1}^H \gets \text{Generator}(z^{\text{start}},\boldsymbol{s},\boldsymbol{a},\tau,\theta)$
\STATE Stack atoms $Z \gets \text{Concat}(z^{(1)},\dots,z^{(H)})$
% \STATE Apply truncation $Z_{\text{tqc}} \gets \text{DropTop}(Z, d\cdot H)$ \COMMENT{optional, consistent with training}
\STATE Output $Q(\boldsymbol{s},\boldsymbol{a}) \gets \text{Mean}(Z{})$
\end{algorithmic}
\end{algorithm}

%% file: table/critic_type.tex
\begin{table}[tb!]
\centering
\caption{Comparison of baseline methods by distributional assumption and generative mechanism.}
\small
\begin{tabular}{lcccc}
\toprule
Method 
& Distribution Assumption 
& Representation 
& Generative Mechanism 
& Aggregation \\ 
\midrule
CDQ 
& Point mass 
& Scalar $Q(s,a)$ 
& Deterministic evaluation
& Min \\

DSAC 
& Gaussian 
& $(\mu,\sigma)$ 
& Gaussian sampling 
& Mean \\

IQN 
& Quantile 
& $F^{-1}(\tau)$ 
& Quantile sampling  
& Mean/Min \\

TQC 
& Quantile 
& \(F^{-1}(\tau_i)\) 
& Quantile sampling 
& Drop-Top \\

VF 
& Flow 
& Velocity field 
& Flow-based sampling
& Mean/Min \\

VDRL 
& Diffusion 
& noise predictor 
& Reverse diffusion 
& Mean/Min \\

DBC (Ours)
& Diffusion Bridge
& End Point $\hat{x}_1(\tau, t)$
& ODE bridge sampling   
& Drop-Top\\

\bottomrule
\end{tabular}
\label{tab:baseline-paradigms}
\end{table}

%% file: table/actor_hyperparameter.tex
\begin{table}[htbp]
\centering
\small
\caption{Actor-side hyperparameters for QVPO, SAC, and TD3. Here, "--" indicates that the corresponding actor does not involve this hyperparameter.}
\label{tab:actor_hyperparams}
\begin{tabular}{lccc}
\toprule
\textbf{Parameter} 
& \textbf{QVPO} 
& \textbf{SAC} 
& \textbf{TD3} \\
\midrule

\multicolumn{4}{l}{\textit{Base Settings}} \\
Discount factor $\gamma$ 
& $0.99$ & $0.99$ & $0.99$ \\
Target update rate $\tau$ 
& $0.005$ & $0.005$ & $0.005$ \\

\midrule
\multicolumn{4}{l}{\textit{Actor Network Architecture}} \\
Hidden units 
& $256$ 
& $256$ 
& $256$ \\
Time embedding dim 
& $32$ 
& -- 
& -- \\
Hidden layers
& $3$
& $3$
& $3$ \\
Activation layer
& Mish
& ReLU
& ReLU\\
\midrule
\multicolumn{4}{l}{\textit{Diffusion Parameterization}} \\
Number of diffusion steps $T$ 
& $20$ 
& -- 
& -- \\
Beta schedule 
& Cosine 
& -- 
& -- \\
Predict noise $\epsilon$ 
& True 
& -- 
& -- \\
Noise ratio 
& $1.0$ 
& -- 
& -- \\

\midrule
\multicolumn{4}{l}{\textit{Optimization (Actor Only)}} \\
Actor learning rate 
& $1\times 10^{-4}$ 
& $3\times 10^{-4}$ 
& $3\times 10^{-4}$ \\
Gradient clipping 
& $2.0$ 
& $1.0$ 
& $1.0$ \\

\midrule
\multicolumn{4}{l}{\textit{Sampling / Policy Evaluation}} \\
Train samples per state 
& $64$ 
& -- 
& -- \\
Eval samples per state 
& $32$ 
& -- 
& -- \\

\midrule
\multicolumn{4}{l}{\textit{Regularization / Entropy}} \\
Entropy coefficient $\alpha$ 
& $0.02$ 
& $1$
& -- \\
Minimum $\alpha$ 
& $0.002$ 
& -- 
& -- \\

\midrule
\multicolumn{4}{l}{\textit{Update Frequency}} \\
Policy Delay
& $1$ 
& $1$ 
& $2$ \\

\bottomrule
\end{tabular}
\end{table}

%% file: table/critic_hyperparameter.tex
\begin{table}[htbp]
\centering
\small
\caption{Common critic hyperparameters across methods.}
\label{tab:critic_hyperparameters}
\setlength{\tabcolsep}{4pt}
\begin{tabular}{lccccccc}
\toprule
\textbf{Parameter} &
\textbf{DSAC-V1} &
\textbf{IQN} &
\textbf{DBC} &
\textbf{TQC} &
\textbf{CDQ} &
\textbf{ValueFlows} &
\textbf{VDRL} \\
\midrule
Network Structure &
MLP &
IQN &
MLP &
MLP &
MLP &
MLP &
MLP \\
Hidden width &
$256$ & $256$ & $512$ & $512$ & $512$ & $512$ & $256$ \\
Activation &
GELU &
ReLU &
ReLU &
ReLU &
ReLU &
ReLU &
Mish \\
Number of heads &
1 &
2 &
2 &
2 &
2 &
2 &
2 \\
Optimizer &
Adam &
Adam &
Adam &
Adam &
Adam &
Adam &
Adam \\
Learning rate &
$3\times 10^{-4}$ &
$3\times 10^{-4}$ &
$3\times 10^{-4}$ &
$3\times 10^{-4}$ &
$3\times 10^{-4}$ &
$3\times 10^{-4}$ &
$3\times 10^{-4}$ \\
Adam $\epsilon$ &
$10^{-5}$ &
$10^{-5}$ &
$10^{-5}$ &
$10^{-5}$ &
$10^{-5}$ &
$10^{-5}$ &
$10^{-8}$ \\
Gradient clip (max norm) &
$1.0$ &
$1.0$ &
$1.0$ &
$1.0$ &
$1.0$ &
$1.0$ &
$1.0$ \\
Discount $\gamma$ &
$0.99$ &
$0.99$ &
$0.99$ &
$0.99$ &
$0.99$ &
$0.99$ &
$0.99$ \\
Target Polyak $\tau$ &
$0.005$ &
$0.005$ &
$0.005$ &
$0.005$ &
$0.005$ &
$0.005$ &
$0.005$ \\
\bottomrule
\end{tabular}
\end{table}

%% file: table/distributional_hyperparameter.tex
\begin{table}[htbp]
\centering
\caption{Comparison of online/target sampling counts and aggregation strategies. IQN aggregates via minimum over sampled quantiles; TQC and DBC apply top value truncation, with DBC using asymmetric online/target sampling. Here, "--" indicates that the corresponding critic does not involve this hyperparameter.}
\small
\setlength{\tabcolsep}{6pt}
\begin{tabular}{lcccc}
\toprule
\textbf{Method} &
\textbf{Aggregation} &
\textbf{Online Samples} &
\textbf{Target Samples} &
\textbf{Embedding Dimension} \\
\midrule
VF &
Mean &
$1$ &
$1$ &
$1$ \\
VDRL &
Min &
$5$ &
$10$ &
$16$ \\
IQN &
Min &
$32$ &
$32$ &
$64$ \\
TQC &
Drop-top &
$25$ &
$25$ &
--- \\
DBC &
Drop-top &
$64$ &
$128$ &
$32$ \\
\bottomrule
\end{tabular}
\label{tab:distribution_hyperparameters}
\end{table}

%% file: table/drop_ratio.tex
\begin{table}[htbp]
\centering
\small
\caption{Drop ratios for DBC and TQC across environments. TQC drop settings strictly follow the original TQC paper, while DBC uses environment-specific truncation ratios under a higher-resolution target distribution.}
\label{tab:drop_ratio}
\begin{tabular}{lcc}
\toprule
\textbf{Environment} &
\textbf{DBC (Drop / Total)} &
\textbf{TQC (Drop / Total)} \\
\midrule
HalfCheetah-v5 &
$0 / 128$ &
$0 / 25$ \\
Ant-v5 &
$12 / 128$ &
$2 / 25$ \\
Walker2d-v5 &
$14 / 128$ &
$2 / 25$ \\
Humanoid-v5 &
$12 / 128$ &
$2 / 25$ \\
Hopper-v5 &
$32 / 128$ &
$5 / 25$ \\
\bottomrule
\end{tabular}
\end{table}

%% file: table/full_experiments.tex
\begin{table}[htbp]
\caption{Experiments Performance for all baselines}
\label{tab:full_experiments}
\begin{center}
\begin{small}
\begin{tabular}{lccccc}
\toprule
Run & Ant-v5 & HalfCheetah-v5 & Hopper-v5 & Humanoid-v5 & Walker2d-v5 \\
\midrule
QVPO+CDQ & 6373.2 \(\pm\) 65.0 & 12103.9 \(\pm\) 422.3 & 3582.0 \(\pm\) 48.1 & 5068.3 \(\pm\) 28.5 & 4986.3 \(\pm\) 279.6 \\
SAC+CDQ & 6121.8 \(\pm\) 412.7 & 12776.4 \(\pm\) 285.2 & 3630.5 \(\pm\) 50.0 & 5207.1 \(\pm\) 94.1 & 4854.4 \(\pm\) 372.8 \\
SAC+DSAC & 2013.2 \(\pm\) 679.9 & 12974.0 \(\pm\) 191.7 & 3522.9 \(\pm\) 23.5 & 3442.6 \(\pm\) 216.8 & 3297.4 \(\pm\) 1248.3 \\
SAC+IQN & 3487.4 \(\pm\) 798.5 & 10930.2 \(\pm\) 152.7 & 3301.8 \(\pm\) 360.1 & 4729.2 \(\pm\) 580.5 & 4766.5 \(\pm\) 255.7 \\
SAC+TQC & 6342.6 \(\pm\) 294.7 & \textbf{13996.2 \(\pm\) 191.7} & 3704.6 \(\pm\) 94.5 & 5269.0 \(\pm\) 113.8 & 5802.6 \(\pm\) 234.2 \\
SAC+VD & 3850.4 \(\pm\) 987.1 & 4284.8 \(\pm\) 478.2 & 1052.7 \(\pm\) 20.4 & 4886.9 \(\pm\) 228.8 & 2533.9 \(\pm\) 302.4 \\
SAC+VF & 2650.3 \(\pm\) 163.8 & 8199.5 \(\pm\) 1533.3 & 3310.4 \(\pm\) 24.6 & 4950.3 \(\pm\) 60.3 & 2626.4 \(\pm\) 1387.4 \\
TD3+CDQ & 4257.0 \(\pm\) 1711.6 & 11679.0 \(\pm\) 629.8 & 3596.2 \(\pm\) 38.8 & 5067.8 \(\pm\) 33.7 & 5093.7 \(\pm\) 439.8 \\
\midrule
SAC+DBC(*) & 6501.4 \(\pm\) 84.3 & 13120.2 \(\pm\) 162.4 & \textbf{3732.5 \(\pm\) 83.2} & \textbf{5906.7 \(\pm\) 198.6} & 6138.2 \(\pm\) 38.1 \\
QVPO+DBC(*) & \textbf{6633.8 \(\pm\) 71.2} & 13182.1 \(\pm\) 252.4 & 3721.3 \(\pm\) 116.6 & 5426.3 \(\pm\) 9.0 & 5448.1 \(\pm\) 193.8 \\
TD3+DBC(*) & 5306.0 \(\pm\) 1303.6 & 13787.8 \(\pm\) 346.5 & 3685.7 \(\pm\) 68.2 & 5343.2 \(\pm\) 199.3 & \textbf{6335.5 \(\pm\) 269.4} \\
\bottomrule
\end{tabular}
\end{small}
\end{center}
\end{table}

%% file: table/schedule.tex
\begin{table}[htbp]
\caption{Ablation for different schedules}
\label{tab:ablation:schedule}
\begin{center}
\begin{small}
\begin{tabular}{lccccc}
\toprule
Schedule & Ant-v5 & HalfCheetah-v5 & Hopper-v5 & Humanoid-v5 & Walker2d-v5 \\
\midrule
Constant & 6501.4 \(\pm\) 84.3 & \textbf{13120.2 \(\pm\) 162.4} & 3732.5 \(\pm\) 83.2 & \textbf{5906.7 \(\pm\) 198.6} & 6138.2 \(\pm\) 38.1 \\
Cosine & 5358.3 \(\pm\) 1628.9 & 12538.9 \(\pm\) 266.1 & \textbf{3748.6 \(\pm\) 144.4} & 5623.8 \(\pm\) 368.7 & \textbf{6326.6 \(\pm\) 751.9} \\
Linear & \textbf{6544.5 \(\pm\) 113.9} & 13004.6 \(\pm\) 174.7 & 3728.6 \(\pm\) 50.9 & 5594.4 \(\pm\) 305.7 & 6116.7 \(\pm\) 517.3 \\
\bottomrule
\end{tabular}
\end{small}
\end{center}
\end{table}

%% file: table/w2_weight_full.tex
\begin{table}[htbp]
\caption{Sensitivity Analysis of Anchor Loss Weight \(w\)}
\label{tab:ablation:weight_full}
\vskip 0.15in
\begin{center}
\begin{small}
\begin{tabular}{lccccc}
\toprule
 & Ant-v5 & HalfCheetah-v5 & Hopper-v5 & Humanoid-v5 & Walker2d-v5 \\
\midrule
\(w=0.0\) & 4828.3 \(\pm\) 1439.0 & 12990.1 \(\pm\) 719.7 & 3646.3 \(\pm\) 84.9 & 5807.3 \(\pm\) 175.8 & 5562.9 \(\pm\) 423.2 \\
\(w=0.01\) & \textbf{6501.4 \(\pm\) 84.3} & 13120.2 \(\pm\) 162.4 & \textbf{3748.6 \(\pm\) 144.0} & \textbf{5906.7 \(\pm\) 198.6} & \textbf{6138.2 \(\pm\) 38.1} \\
\(w=0.02\) & 3947.3 \(\pm\) 2151.8 & \textbf{13165.8 \(\pm\) 301.2} & 3702.6 \(\pm\) 86.5 & 5743.4 \(\pm\) 301.1 & 6097.8 \(\pm\) 375.1 \\
\(w=0.1\) & 5716.2 \(\pm\) 741.3 & 12775.5 \(\pm\) 127.8 & 3591.1 \(\pm\) 79.2 & 5589.0 \(\pm\) 50.3 & 5974.6 \(\pm\) 603.2 \\
No Eq.~\eqref{eq:loss:quantile_loss} & 2517.1 \(\pm\) 2464.7 & 12393.1 \(\pm\) 314.9 & 3244.0 \(\pm\) 497.5 & 3751.2 \(\pm\) 321.1 & 3859.3 \(\pm\) 1828.7 \\
\bottomrule
\end{tabular}

\end{small}
\end{center}
\end{table}

%% file: table/timesteps_full.tex
\begin{table}[htbp]
\caption{Ablation of Flow Steps \(M\)}
\label{tab:ablation:timesteps_full}
\vskip 0.15in
\begin{center}
\begin{small}
\begin{tabular}{lccccc}
\toprule
 & Ant-v5 & HalfCheetah-v5 & Hopper-v5 & Humanoid-v5 & Walker2d-v5 \\
\midrule
M=1 & 5150.4 \(\pm\) 2226.3 & 12821.5 \(\pm\) 26.4 & \textbf{3748.1 \(\pm\) 46.8} & 5524.7 \(\pm\) 132.7 & 5809.8 \(\pm\) 308.6 \\
M=5 & \textbf{6501.4 \(\pm\) 84.3} & \textbf{13120.2 \(\pm\) 162.4} & \textbf{3748.6 ± 144.4} & \textbf{5906.7 \(\pm\) 198.6} & \textbf{6138.2 \(\pm\) 38.1} \\
M=10 & 5946.6 \(\pm\) 379.7 & 13066.6 \(\pm\) 203.0 & 3607.9 \(\pm\) 30.5 & 4848.7 \(\pm\) 88.1 & 6014.3 \(\pm\) 561.6 \\
\bottomrule
\end{tabular}
\end{small}
\end{center}
\end{table}

%% file: table/training_times.tex
\begin{table}[h!]
\centering
\caption{
Training time (in hours) required by different critics to complete 1M environment steps.
All experiments use SAC as the actor backbone and are evaluated on MuJoCo v5 benchmarks.
}
\label{tab:training_time}
\small
\begin{tabular}{lccccc}
\toprule
\textbf{Critic} & \textbf{Ant-v5} & \textbf{HalfCheetah-v5} & \textbf{Hopper-v5} & \textbf{Walker2d-v5} & \textbf{Humanoid-v5} \\
\midrule
CDQ  & 3.0 & 3.0 & 3.0 & 3.0 & 3.5 \\
TQC  & 3.1 & 3.0 & 3.3 & 3.3 & 3.5 \\
IQN  & 4.0 & 4.0 & 4.0 & 4.0 & 3.8 \\
DBC  & 7.0 & 7.0 & 7.0 & 7.0 & 8.0 \\
\bottomrule
\end{tabular}
\end{table}